\documentclass{article}

\PassOptionsToPackage{square,numbers,sort&compress}{natbib}
\usepackage[main,nonatbib,final]{neurips_2026}

\usepackage[utf8]{inputenc} % allow utf-8 input
\usepackage[T1]{fontenc}    % use 8-bit T1 fonts
\usepackage{hyperref}       % hyperlinks
\usepackage{url}            % simple URL typesetting
\usepackage{booktabs}       % professional-quality tables
\usepackage{amsfonts}       % blackboard math symbols
\usepackage{nicefrac}       % compact symbols for 1/2, etc.
\usepackage{microtype}      % microtypography
\usepackage{xcolor}         % colors

\usepackage{graphicx}
\usepackage{amssymb}
\usepackage{float}
\usepackage{amsmath}
\usepackage{multirow}
\usepackage{placeins}
\hypersetup{
    colorlinks,
    linkcolor={blue!64!black},
    citecolor={blue!64!black},
    urlcolor={blue!64!black}
}
\usepackage{natbib}
\usepackage{amsthm}
\newtheorem{proposition}{Proposition}
\usepackage{booktabs}
\usepackage{marvosym}

\title{A latent control model for realistic rodent motion}

\author{%
  Aidan Sirbu$^{1,2,3,\text{\Letter}}$ \quad
  Charles Y. Zhang$^{4}$ \quad
  Yuanjia Yang$^{6,7}$ \\
  \textbf{Bence P. \"Olveczky$^{5}$ \quad
  Talmo D. Pereira$^{6}$ \quad
  Blake A. Richards$^{1,2,3,8,9,10}$} \\[6pt]
  \small
  \begin{tabular}{c}
    $^{1}$Google, Paradigms of Intelligence Team \enspace
    $^{2}$Mila -- Quebec AI Institute \\
    $^{3}$School of Computer Science, McGill University \\
    $^{4}$Harvard University \enspace
    $^{5}$Department of Organismal and Evolutionary Biology, Harvard University \\
    $^{6}$Salk Institute for Biological Studies \enspace
    $^{7}$University of California San Diego \\
    $^{8}$Department of Neurology and Neurosurgery, McGill University \\
    $^{9}$Montreal Neurological Institute, McGill University \\
    $^{10}$Learning in Machines and Brains Program, CIFAR \\[4pt]
    $^{\text{\Letter}}$\href{mailto:aidan.sirbu@mail.mcgill.ca}{\texttt{aidan.sirbu@mail.mcgill.ca}} \\
    \href{https://aidansirbu.com/SCAMPER-site/}{\texttt{https://aidansirbu.com/SCAMPER-site/}}
  \end{tabular}
}

\begin{document}

\maketitle

\begin{abstract}
The development of embodied models of animal behavior is a major goal of the neuro-AI community. Embodied models will enhance our ability to make direct comparisons between biological and artificial agents. To make such comparisons effective, though, we need embodied models that actually move like animals. Previous work has developed virtual animal bodies and techniques for training realistic motion through imitation. However, to date, there is no system for embodied animal modeling that produces natural movements both during trained goal-directed behavior, and during exploration driven by random generation. Here, we introduce \textbf{S}imulated \textbf{C}ontrol via \textbf{A}ugmented \textbf{M}otor \textbf{P}riors for \textbf{E}mbodied \textbf{R}einforcement learning (SCAMPER), a system for training an embodied virtual rat model, enabling natural movements during both goal-directed behavior and random exploration. SCAMPER uses a variational architecture with two distinct priors, a motion prior for control and a behavior prior for exploration. We show that combining both of these priors allows SCAMPER to generate realistic behavior during exploration and exploit the task structure effectively. SCAMPER can also learn from sparse rewards as a result of the diverse movements induced by the behavior prior, allowing the system to model realistic sparse-reward tasks from experimental neuroscience. Altogether, SCAMPER provides the neuro-AI community with an effective new system for training embodied models of rodents on realistic tasks.
\end{abstract}

%%%%%%%%%%%%%%%%%%%%%%%%%%%%%%%%%%%%%%%%%%%%%%%%%%%%%%%%%%%%%%%%%%%%%%%%%%%%%%%%%%%%%%%
%%%%%%%%%%%%%%%%%%%%%%%%%%%%%%%%%%%%%%%%%%%%%%%%%%%%%%%%%%%%%%%%%%%%%%%%%%%%%%%%%%%%%%%
\section{Introduction}

One of the ultimate goals of neuro-AI research is to be able to fully capture the richness of animal learning and behavior with embodied computational models \cite{segado_grounding_2025, zador_catalyzing_2023}. Previous work has developed realistic virtual bodies for a variety of model organisms commonly used in neuroscience experiments \cite{wang-chen_embodied_2026}. However, controlling a realistic virtual body to consistently produce natural motion is non-trivial. For example, when an artificial neural network was trained to control a virtual rat on a variety of tasks using deep reinforcement learning (RL), the model could solve the tasks, but often did so with unnatural movements that did not resemble the motion of real rats \cite{merel_deep_2019}. While researchers have since used imitation learning to improve the realism of virtual animal motion, these models are either limited to pure imitation \cite{aldarondo_virtual_2024}, or struggle to transfer to downstream RL tasks \cite{zhang_mimic-mjx_2025}.

Within robotics, physics-based controllers trained with RL to produce realistic humanoid motion can generalize to new tasks \cite{peng_deepmimic_2018, merel_neural_2019, peng_amp_2021, peng_ase_2022,hasenclever_comic_2020, tessler_calm_2023}. However, most approaches have no mechanism for random generation of realistic motion. One exception is the recent Physics-based Universal motion Latent SpacE (PULSE) framework that uses residual control of a variational motion prior \cite{luo_universal_2024}. PULSE can sample from its motion prior to produce random, yet realistic movements. However, this sampling exhibits limited diversity: depending on initialization, the system typically settles into a single behavioral mode and stays there. This is especially problematic when PULSE is used as a low-level controller for sparse-reward tasks, which demand substantial structured exploration \cite{pathak_curiosity-driven_2017, ecoffet_first_2021}. Because most experimental tasks in neuroscience offer only sparse extrinsic reward \cite{carandini_probing_2013, the_international_brain_laboratory_standardized_2021, musall_single-trial_2019}, training embodied animal models to solve them while maintaining naturalistic motion remains an open problem.

Here, we present \textbf{S}imulated \textbf{C}ontrol via \textbf{A}ugmented \textbf{M}otor \textbf{P}riors for \textbf{E}mbodied \textbf{R}einforcement learning (SCAMPER), a system for training an embodied virtual rat model, enabling natural movements during both goal-directed behavior and random exploration. SCAMPER builds on the PULSE framework \cite{luo_universal_2024}, but adds an additional behavior prior that greatly increases the diversity of randomly generated movements while preserving the realism of the motion. We provide a formal analysis demonstrating why the behavior prior increases the diversity of generated movements. We then show empirically that SCAMPER produces a diverse array of realistic behavior during random rollouts, covering the behavioral repertoire present in the training data and beyond. We also demonstrate that SCAMPER transfers effectively to a range of downstream tasks, from dense-reward locomotion and rearing to sparse-reward foraging tasks, all while maintaining naturalistic motion throughout training. Altogether, SCAMPER provides the neuro-AI community with a powerful new tool for embodied modeling of animal behavior.

%%%%%%%%%%%%%%%%%%%%%%%%%%%%%%%%%%%%%%%%%%%%
% Figure 1: Overview of SCAMPER 
%%%%%%%%%%%%%%%%%%%%%%%%%%%%%%%%%%%%%%%%%%%%
\begin{figure}[htbp!]
  \centering
  \includegraphics[width=0.9\textwidth]{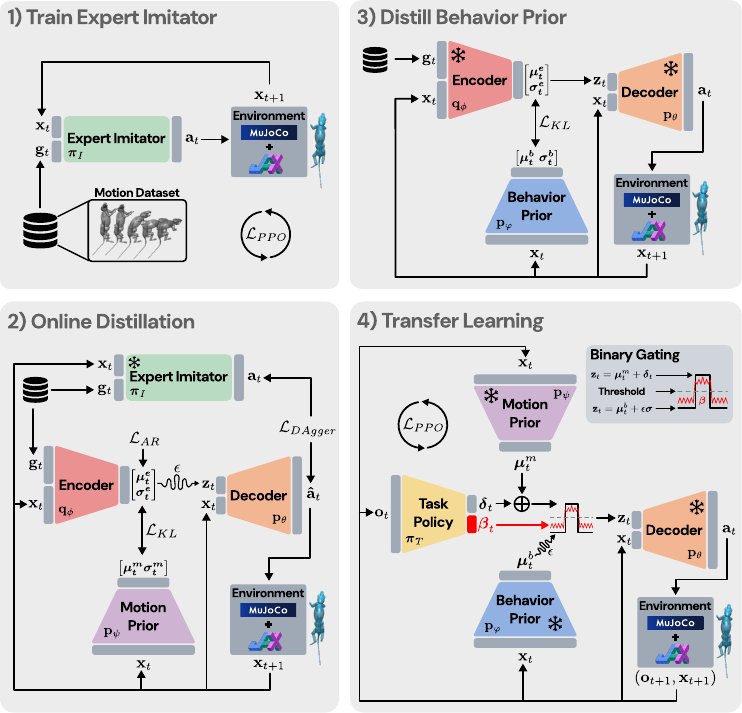}
  \caption{\textbf{Training stages of SCAMPER.} (1) The expert imitator is trained with mocap data. (2) The latent controller is trained via online distillation. (3) The behavior prior is distilled from the frozen encoder. Finally, (4) the model can be trained with RL to perform different tasks.}
  \label{fig:overview}
\end{figure}
%%%%%%%%%%%%%%%%%%%%%%%%%%%%%%%%%%%%%%%%%%%%

%%%%%%%%%%%%%%%%%%%%%%%%%%%%%%%%%%%%%%%%%%%%%%%%%%%%%%%%%%%%%%%%%%%%%%%%%%%%%%%%%%%%%%%
%%%%%%%%%%%%%%%%%%%%%%%%%%%%%%%%%%%%%%%%%%%%%%%%%%%%%%%%%%%%%%%%%%%%%%%%%%%%%%%%%%%%%%%
\section{Related work}

\textbf{Motor priors and latent skill spaces for physics-based control.} Training physics-based controllers from scratch for different tasks often yields unnatural motion \cite{heess_emergence_2017}, motivating imitation-based controllers that track reference motion data \cite{peng_deepmimic_2018}. \citet{merel_neural_2019} extended this approach by distilling a tracking expert into a latent skill space through an information bottleneck, and subsequent work explored adversarial alternatives \cite{peng_amp_2021, peng_ase_2022, tessler_calm_2023} and variational formulations \citep{yao_controlvae_2022, won_physics-based_2022, yao_moconvq_2023}. PULSE \cite{luo_universal_2024} distills a universal motion imitator \cite{luo_perpetual_2023} into a variational model regularized with a learnable motion prior, achieving broad coverage of human motor skills. These frameworks target humanoid characters trained on human motion capture (mocap) data; their applicability to non-human morphologies driven by species-specific behavioral data has not yet been demonstrated. Moreover, the learned motion prior in PULSE supports local motion continuation, but random rollouts tend to remain within the behavioral mode in which they are initialized, producing locally plausible but behaviorally narrow motion. Kinematic motion models, including autoregressive CVAE-based priors \cite{rempe_humor_2021} and diffusion-based generators \cite{tevet_human_2022}, can produce diverse long-horizon motion but do not yield a physics-based controller. Furthermore, autoregressive policies have been applied to physics-based control \cite{wagener_mocapact_2023,radosavovic_humanoid_2024}, but they provide no mechanism for sampling diverse motion that respects a species-specific behavioral repertoire. Here, we extend the PULSE framework to a biomechanically realistic rodent trained on 3D kinematic recordings of freely behaving rats, and complement the motion prior with a behavior prior that induces realistic transitions across behavioral modes.

\textbf{Physics-based control of virtual animals.} Biomechanically realistic virtual animals have emerged as a tool for neuro-AI. \citet{merel_deep_2019} introduced a virtual rodent trained with RL, but without imitation pretraining the resulting motion was unnatural. \citet{aldarondo_virtual_2024} extended this work by training a controller for the virtual rat to imitate freely moving rats and found that its internal representations matched those in the motor cortex and dorsolateral striatum of the real rats; the controller was, however, purely imitative and could not generate diverse behavior. MIMIC-MJX \cite{zhang_mimic-mjx_2025} generalizes this tracking-imitation approach across species via an open-source JAX/MJX pipeline, but supports only limited transfer learning. Parallel efforts in \textit{Drosophila} \cite{vaxenburg_whole-body_2025, wang-chen_neuromechfly_2024} and \textit{C. elegans} \cite{zhao_integrative_2024} further extend this paradigm beyond rodents. In zebrafish, \citet{keller_autonomous_2025} drive naturalistic behavioral transitions via model-based intrinsic motivation against an ethological prior. Across this line of work, no existing model provides the ability to both mimic real motion and support sampling appropriate for exploration. Our work fills this gap, yielding a controller that can be both reused across tasks and sampled for diverse, naturalistic behavior.

%%%%%%%%%%%%%%%%%%%%%%%%%%%%%%%%%%%%%%%%%%%%%%%%%%%%%%%%%%%%%%%%%%%%%%%%%%%%%%%%%%%%%%%
%%%%%%%%%%%%%%%%%%%%%%%%%%%%%%%%%%%%%%%%%%%%%%%%%%%%%%%%%%%%%%%%%%%%%%%%%%%%%%%%%%%%%%%
\section{Overview of SCAMPER}

\subsection{Preliminaries}
\label{sec:preliminaries}
We define the proprioception of the agent at a given timestep as 
\begin{equation}
    \mathbf{x}_t = (\boldsymbol{\theta}_t, \boldsymbol{\dot{\theta}}_t, \boldsymbol{\tau}_t, h_t, \boldsymbol{\phi}_t, \mathbf{p}^{ee}_t, \mathbf{s}_t^{kin}, \mathbf{c}_t, \mathbf{a}_{t-1}),
\end{equation}
containing the joint angles ($\boldsymbol{\theta}_t$), angular velocities ($\boldsymbol{\dot{\theta}}_t$), joint torques ($\boldsymbol{\tau}_t$), body height ($h_t$), orientation relative to the world z-axis ($\boldsymbol{\phi}_t$, akin to a vestibular signal), egocentric end-effector positions ($\mathbf{p}^{ee}_t$), kinematic sensors (accelerometer, velocimeter, gyro; $\mathbf{s}_t^{kin}$), end-effector touch sensors ($\mathbf{c}_t$), and the previous action ($\mathbf{a}_{t-1}$). During imitation tasks, the goal of the agent at a given timestep is defined as
\begin{equation}
    \mathbf{g}_t = (\mathbf{r}_{t+1}-\mathbf{r}_t, \mathbf{q}_{t+1}\ominus\mathbf{q}_t, \boldsymbol{\theta}_{t+1}-\boldsymbol{\theta}_t, \mathbf{p}^{seg}_{t+1}-\mathbf{p}^{seg}_{t}),
\end{equation}
and specifies the displacement the agent must achieve, computed as differences between its current state and the target state at the next timestep in the root position ($\mathbf{r}_t$), root orientation ($\mathbf{q}_t$; where $\ominus$ is egocentric quaternion difference), joint angles ($\boldsymbol{\theta}_t$), and egocentric Cartesian body segment positions ($\mathbf{p}^{seg}_t$). Target states are provided by a mocap dataset registered to the biomechanical rat model via inverse kinematics using STAC-MJX \cite{zhang_mimic-mjx_2025}. During imitation the policy network is provided the next five goal states ($g_{t:t+4}$). During transfer learning, the agent no longer receives a goal state, $\mathbf{g}_t$, and instead observes task-specific environmental observations $\mathbf{o}_t$ (e.g., scent gradient). In this work we use the terms `rat', `rodent', and `agent' interchangeably.

We frame each task (aside from the distillation objectives) as goal-conditioned RL, formulated as a Markov Decision Process $\mathcal{M}=\left< \mathcal{S}, \mathcal{A}, \mathcal{T}, \mathcal{R}, \gamma \right>$. The states $s_t\in\mathcal{S}$ contain $\mathbf{x}_t$ and either $\mathbf{g}_t$ or $\mathbf{o}_t$; all values are normalized by their running mean and variance. Actions $\mathbf{a}_t\in\mathcal{A}$ are torques applied to the rodent's actuated joints. $\mathcal{T}$ and $\mathcal{R}$ are the transition dynamics (determined by the physics environment) and the reward function, respectively. We train policies with proximal policy optimization (PPO) \cite{schulman_proximal_2017} to maximize the discounted return $\mathbb{E}\left[\sum_{t=1}^T\gamma^{t-1}r_t\right]$.

%%%%%%%%%%%%%%%%%%%%%%%%%%%%%%%%%%%%%%%%%%%%%%%%%%%%%%%%%%%%%%%%%%%%%%%%%%%%%%%%%%%%%%%
\subsection{Training pipeline}
SCAMPER uses a four-stage training pipeline (Fig.~\ref{fig:overview}). All networks are multi-layer perceptrons (MLPs), but in principle other architectures could be used.

\textbf{(1) Expert imitator.} We train an expert imitator $\pi_I(\mathbf{a}_t|\mathbf{x}_t,\mathbf{g}_t)$ with PPO to reproduce the mocap trajectories (SuppVid1; App.~\ref{app:supvids}).

\textbf{(2) Online distillation.} Following the PULSE framework \cite{luo_universal_2024}, we distill the expert imitator into a latent skill controller. Similar to PULSE, we find that training the latent skill controller using supervised distillation of a pre-trained expert leads to better training stability than training the latent skill controller from scratch using RL, hence we adopt the same approach. The controller consists of three jointly trained networks: an encoder $q_\phi(\mathbf{z}_t|\mathbf{x}_t,\mathbf{g}_t)$, a decoder $p_\theta(\mathbf{a}_t|\mathbf{x}_t,\mathbf{z}_t)$, and a proprioception-conditioned motion prior $p_\psi(\mathbf{z}_t|\mathbf{x}_t)$ over latent codes. Training proceeds by rolling out the student controller on the imitation task in simulation, while querying the frozen expert, $\pi_I$, for target actions on visited states, following DAgger \cite{ross_reduction_2011}. The networks are jointly trained by minimizing:
\begin{equation}
\label{eq:objective}
    \mathcal{L}_{\mathrm{distill}} = \mathcal{L}_{\mathrm{DAgger}} + \lambda\mathcal{L}_{\mathrm{KL}} + \alpha\mathcal{L}_{\mathrm{AR}},
\end{equation}
where $\mathcal{L}_{\mathrm{DAgger}}=||\mathbf{a}_t-\mathbf{\hat{a}}_t||_2^2$ matches the expert's target action $\mathbf{a}_t$ to the decoder's prediction $\mathbf{\hat{a}}_t$, $\mathcal{L}_{\mathrm{KL}}=D_{\mathrm{KL}}(q_\phi||p_\psi)$ regularizes the encoder toward the motion prior, and $\mathcal{L}_{\mathrm{AR}}=||\boldsymbol{\mu}_t^e-\boldsymbol{\mu}_{t-1}^e||_2^2$ encourages temporal smoothness in the latent space. The decoder and motion prior are retained for downstream use; the encoder is kept only for the next training stage. $\lambda$ and $\alpha$ are hyperparameters (Table~\ref{tab:distill_hparams}).

\textbf{(3) Behavior prior distillation.} After online distillation, we train a second proprioception-conditioned prior $p_\varphi(\mathbf{z}_t|\mathbf{x}_t)$ with all other networks frozen. Rolling out the latent skill controller on the imitation task in simulation, we optimize $p_\varphi$ by minimizing
\begin{equation}
    \mathcal{L}_{\mathrm{behavior}} = D_{\mathrm{KL}}(q_\phi(\mathbf{z}_t | \mathbf{x}_t, \mathbf{g}_t) \,\|\, p_\varphi(\mathbf{z}_t | \mathbf{x}_t)).
\end{equation}
By Proposition~\ref{prop:behavior_prior}, this ensures $p_\varphi$ captures the full goal-conditional variation in the encoder's outputs.

\textbf{(4) Transfer learning.} For downstream tasks, a task policy $\pi_T(\boldsymbol{\delta}_t,\beta_t|\mathbf{x}_t,\mathbf{o}_t)$ is trained with PPO to produce a residual correction $\boldsymbol{\delta}_t$ to the motion prior's mean $\boldsymbol{\mu}_t^m$. Importantly, we adjudicate between the two priors (motion and behavior) using a continuous gating variable $\beta_t\in[-1,1]$ that is binarized to $b_t\in\{0,1\}$ via hysteresis thresholding. Thus, at each timestep, $b_t$ selects between two latent codes:
\begin{equation}
    \mathbf{z}_t = \begin{cases}
        \boldsymbol{\mu}_t^m + \boldsymbol{\delta}_t & \text{if } b_t = 1, \\
        \boldsymbol{\mu}_t^b + \boldsymbol{\epsilon}\boldsymbol{\sigma}, \quad \boldsymbol{\epsilon} \sim \boldsymbol{\mathcal{N}(0, 1)} & \text{if } b_t = 0,
    \end{cases}
\end{equation}
where $\boldsymbol{\mu}_t^m$ and $\boldsymbol{\mu}_t^b$ are the means of $p_\psi(\mathbf{z}_t|\mathbf{x}_t)$ and $p_\varphi(\mathbf{z}_t|\mathbf{x}_t)$, respectively, and $\boldsymbol{\sigma}$ is a fixed noise scale. The resulting latent code is passed through the frozen decoder $p_\theta$ to produce the action. Intuitively, $b_t=1$ engages goal-directed residual control of the motion prior (exploitation), while $b_t=0$ samples naturalistic motion from the behavior prior (exploration). Because $\pi_T$ outputs $\beta_t$, the system learns when to explore and when to exploit.

%%%%%%%%%%%%%%%%%%%%%%%%%%%%%%%%%%%%%%%%%%%%%%%%%%%%%%%%%%%%%%%%%%%%%%%%%%%%%%%%%%%%%%%
\subsection{Formal analysis of the priors}
\label{sec:priors_analysis}
Why does SCAMPER use both a motion prior, $p_\psi$, for goal-directed motion and a behavior prior, $p_\varphi$, for exploration? Although both $p_\psi$ and $p_\varphi$ are proprioception-conditioned distributions over the same latent space, their training objectives differ. We formalize this distinction in two propositions.
\begin{proposition}[Information bottleneck under joint training]\label{prop:motion_prior}
When the encoder $q_\phi(\mathbf{z}|\mathbf{g},\mathbf{x})$ and the state-dependent prior $p_\psi(\mathbf{z}|\mathbf{x})$ are trained jointly, minimizing their expected KL divergence is equivalent to the following optimization:
\begin{equation}
    \min_{\phi, \psi} \mathbb{E}_{p(\mathbf{x},\mathbf{g})}\left[D_{\mathrm{KL}}(q_\phi \parallel p_\psi)\right] = \min_{\phi, \psi} \left[ I_\phi(Z; G \mid X) + \mathbb{E}_{p(\mathbf{x})} \left[ D_{\mathrm{KL}}(m_\phi \parallel p_\psi) \right] \right],
\end{equation}
where $m_\phi(\mathbf{z}|\mathbf{x})=\mathbb{E}_{p(\mathbf{g}|\mathbf{x})}[q_\phi(\mathbf{z}|\mathbf{g},\mathbf{x})]$ is the marginal posterior induced by the encoder, which evolved during joint training.
\end{proposition}
Joint training of the encoder and the motion prior thus not only forces the motion prior to approximate the marginal posterior, but also minimizes the encoder's conditional mutual information between the latent states and the goals. However, from equation (\ref{eq:objective}) both of these terms which the KL loss decomposes into are weighted by the same hyperparameter ($\lambda$). As a result, in practice, it is difficult to obtain a favorable balance between regularization via the bottleneck and the distortion of movement imitation. We find that setting $\lambda$ in order to obtain the right balance limits how closely the motion prior can approximate the marginal posterior, yielding a motion prior concentrated on local motion continuation. We also find that stochastic sampling of the motion prior with high noise can occasionally push the system to switch behavioral modes, but as we show below (Fig.~\ref{fig:behavior-diversity}), the motion prior alone does not reproduce the rich pattern of behavioral transitions observed in the training data. Despite all this, a motion prior focused on local motion continuation yields a network favorable for residual control. The marginal posterior, by contrast, contains the more fine-grained structure of the behavioral transitions. Therefore, in order to have the practical ability to generate diverse behavior during exploration we need another prior that \textit{only approximates the marginal posterior}, allowing us to decouple the training dynamics from minimizing the conditional mutual information. We can achieve this by training a new prior, the behavior prior, with the frozen encoder.
\begin{proposition}[Marginal posterior approximation under frozen encoder]\label{prop:behavior_prior}
When the encoder $q_\phi(\mathbf{z}|\mathbf{g},\mathbf{x})$ is frozen, minimizing the expected KL divergence between the encoder and a state-dependent prior $p_\varphi(\mathbf{z}|\mathbf{x})$ is equivalent to
\begin{equation}
    \min_\varphi \mathbb{E}_{p(\mathbf{x},\mathbf{g})}\left[D_{\mathrm{KL}}(q_\phi \parallel p_\varphi)\right] = I_\phi(Z; G \mid X)+ \min_\varphi \mathbb{E}_{p(\mathbf{x})}\left[D_{\mathrm{KL}}(m_\phi \parallel p_\varphi)\right],
\end{equation}
where $m_\phi(\mathbf{z}|\mathbf{x}) = \mathbb{E}_{p(\mathbf{g}|\mathbf{x})}[q_\phi(\mathbf{z}|\mathbf{g},\mathbf{x})]$ is the marginal posterior induced by the frozen encoder, held fixed during optimization. The conditional mutual information, $I_\phi(Z; G \mid X)$, does not depend on the optimization variable $\varphi$; it is therefore an additive constant that sits outside the minimization. Hence, the optimal behavior prior is the closest possible state-conditional approximation to the marginal posterior $m_\phi$.
\end{proposition}
Therefore, by directly distilling the behavior prior from the frozen pre-trained encoder, we are able to lower the KL divergence between it and the goal-rich encoder to a greater extent than we could for the motion prior which was limited by the rate-distortion tradeoff. Specifically, we achieve a KL divergence of 5.88 nats between the behavior prior and the encoder, vs 14.35 nats for the motion prior and the encoder. This allows the behavior prior to better capture the behavioral transitions in the training data. However, since the behavior prior is more inclined to sample from transitions in behavior in its control signal, it is less stable when residually controlled than the motion prior (Fig.~\ref{fig:supp-gait}, Table~\ref{tab:gait-full}). This motivates the use of the dual priors for a sparse-reward exploration task in Section~\ref{sec:forage}. We provide proofs for these two propositions in App.~\ref{app:proofs} and empirical justification in App.~\ref{app:emp-just}. By training the behavior prior in this way, we achieve a good approximation for the marginal posterior, $m_\phi(\mathbf{z}|\mathbf{x})$, which can then be used for effective exploration with natural movements, as we show in the subsequent experiments.

%%%%%%%%%%%%%%%%%%%%%%%%%%%%%%%%%%%%%%%%%%%%%%%%%%%%%%%%%%%%%%%%%%%%%%%%%%%%%%%%%%%%%%%
%%%%%%%%%%%%%%%%%%%%%%%%%%%%%%%%%%%%%%%%%%%%%%%%%%%%%%%%%%%%%%%%%%%%%%%%%%%%%%%%%%%%%%%
\section{Experiments}

\FloatBarrier
\subsection{The behavior prior expands behavioral diversity beyond the motion prior}
\label{sec:behavior-diversity}

%%%%%%%%%%%%%%%%%%%%%%%%%%%%%%%%%%%%%%%%%%%%
% Figure 2: Random generation diversity
% A: Renderings of rollouts
% B: Vendi score plot
% C: Violin plots of net displacement
% D: Trajectory heatmaps
%%%%%%%%%%%%%%%%%%%%%%%%%%%%%%%%%%%%%%%%%%%%
\begin{figure}[!htbp]
  \centering
  \includegraphics[width=0.9\textwidth]{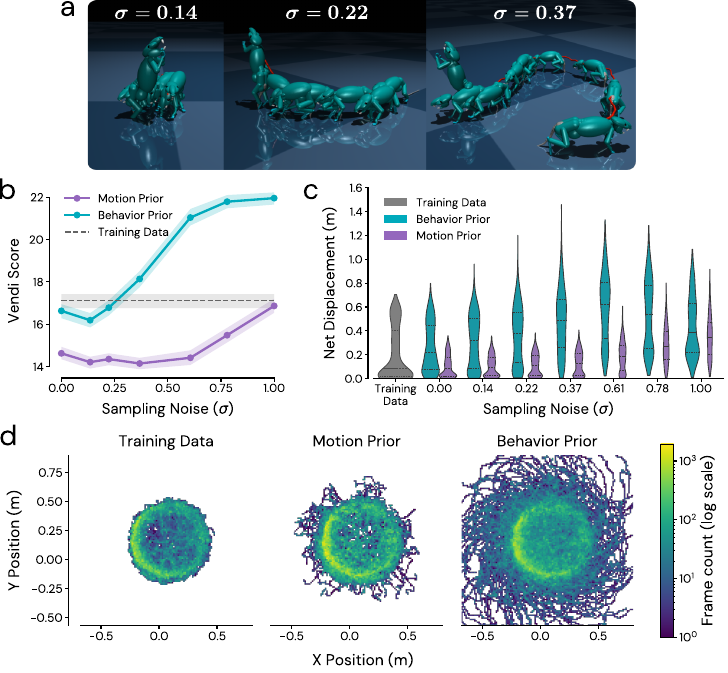}
  \caption{\textbf{SCAMPER is behaviorally diverse.} \textbf{(a)} Renderings of behavior-prior rollouts at increasing sampling noise. \textbf{(b)} Vendi score vs.\ latent sampling noise $\sigma=\exp(\text{logvar}/2)$, for the behavior prior (\textit{blue}), motion prior (\textit{purple}), and the training data (\textit{dashed gray}); data shown is bootstrapped means with 95\% CIs. \textbf{(c)} Per-rollout net displacement across $\sigma$; violins with interquartile lines, gray violin marks training data. \textbf{(d)} Spatial occupancy at $\sigma=0.22$ shown as 2D histograms of body location (log-scaled frame counts) for the training data, motion prior, and behavior prior (see Fig.~\ref{fig:supp-heatmaps} for all noise levels).}
  \label{fig:behavior-diversity}
\end{figure}

We first verify Propositions \ref{prop:motion_prior} and \ref{prop:behavior_prior}: that sampling the behavior prior produces greater behavioral diversity than sampling the motion prior (SuppVid[2-4]; App.~\ref{app:supvids}). We initialized the rodent in a realistic pose, then iteratively sampled a latent from one prior, decoded it to an action, and stepped the simulation for 5\,s of simulated time per rollout (Fig.~\ref{fig:behavior-diversity}a). We collected 842 rollouts per noise level, comparing pure behavior-prior sampling ($b_t=0$) to pure motion-prior sampling ($b_t=1$); see App.~\ref{app:rollout-collection}.

We measured the Vendi score \cite{friedman_vendi_2023}, which captures the effective number of distinct behavioral modes. The Vendi score has been used to quantify trajectory and skill diversity in RL \cite{stanic_beyond_2026, lintunen_vendirl_2025}. Across all noise levels, the Vendi score was significantly higher for the behavior prior than the motion prior (Fig.~\ref{fig:behavior-diversity}b). Motion-prior sampling (similar to PULSE) produced lower Vendi scores than the training data for all $\sigma<1$. At very high noise, motion-prior diversity matched the training data but at the cost of motion realism (Section~\ref{sec:behavior-real}). In contrast, the behavior prior matched the training data diversity at low noise and exceeded it at high noise. This provides a tunable knob for controlling the behavioral diversity of the agent.

Next, to assess whether this diversity aids exploration, we measured the net distance traveled across sampling noise levels. Mirroring the Vendi score results, motion-prior sampling produced shorter trajectories than the training data at almost all noise levels (Fig.~\ref{fig:behavior-diversity}c). Behavior-prior sampling roughly matched the training data distance at low noise and exceeded it at high noise. We also examined the spatial occupancy in the virtual arena. In the training data, the rat was in a circular arena, and thus, the occupancy was restricted to a circular area (Fig.~\ref{fig:behavior-diversity}d, \textit{left}). Motion-prior rollouts at low noise remained largely within this same area despite no physical barrier (Fig.~\ref{fig:behavior-diversity}d, \textit{middle}), demonstrating that the motion prior samples from a restricted distribution of movements seen in the training data. The behavior prior, in contrast, explored beyond this region (Fig.~\ref{fig:behavior-diversity}d, \textit{right}).

%%%%%%%%%%%%%%%%%%%%%%%%%%%%%%%%%%%%%%%%%%%%%%%%%%%%%%%%%%%%%%%%%%%%%%%%%%%%%%%%%%%%%%%
\FloatBarrier
\subsection{The behavior prior preserves motion realism under stochastic sampling}
\label{sec:behavior-real}

%%%%%%%%%%%%%%%%%%%%%%%%%%%%%%%%%%%%%%%%%%%%
% Figure 3: Random generation naturalism 
% A: KID for behavior prior vs motion prior
% B: Root height ECDF
%%%%%%%%%%%%%%%%%%%%%%%%%%%%%%%%%%%%%%%%%%%%
\begin{figure}[!htbp]
  \centering
  \includegraphics[width=\textwidth]{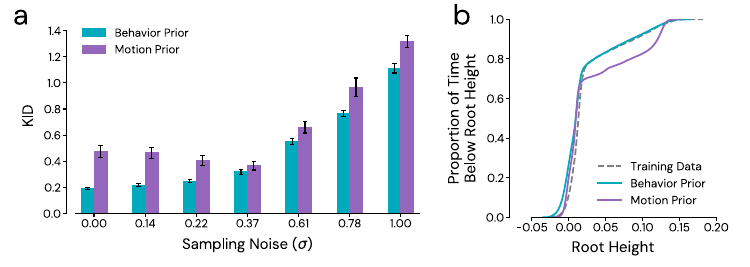}
  \caption{\textbf{SCAMPER preserves motion realism during random behavior.} \textbf{(a)} Kernel Inception Distance (KID) between random rollouts and the training data behavior distributions across sampling noise $\sigma=\exp(\text{logvar}/2)$, for the behavior prior (\textit{blue}) and motion prior (\textit{purple}). Data shown is means $\pm$1 s.d. \textbf{(b)} Empirical cumulative distribution function (ECDF) of root height at $\sigma=0.22$, compared against the training data (\textit{dashed gray}).}
  \label{fig:behavior-realism}
\end{figure}

To verify that the movements induced by the behavior prior during random exploration were natural we used the Kernel Inception Distance (KID), which measures the dissimilarity between real and generated probability distributions \cite{binkowski_demystifying_2021} and has been used to measure the realism of generative human motion \cite{raab_modi_2022}. We trained a variational autoencoder (VAE) on the joint angle trajectories from the training data as the inception model, then computed the KID between the VAE features of the training and generated trajectories. Across all noise levels, the KID was lower for the behavior prior than the motion prior (Fig.~\ref{fig:behavior-realism}a). This confirms that the behavior prior produces more natural movements during random generation than the motion prior.

For a more concrete measure of motion realism, we examined the rearing behavior of the virtual rat by measuring the proportion of time spent at various body heights. The behavior prior more closely matched the training data's rearing distribution, while the motion prior diverged from it (Fig.~\ref{fig:behavior-realism}b). This held across different levels of noise (Fig.~\ref{fig:supp-rearing}), and was quantitatively confirmed using the Wasserstein distance of the rearing curves between the generated data and the training data (Table~\ref{tab:wasserstein-rear}). Together with Section~\ref{sec:behavior-diversity}, these results show that the behavior prior can generate diverse and realistic rat motion during random sampling.

%%%%%%%%%%%%%%%%%%%%%%%%%%%%%%%%%%%%%%%%%%%%%%%%%%%%%%%%%%%%%%%%%%%%%%%%%%%%%%%%%%%%%%%
\FloatBarrier
\subsection{The motion prior helps produce realistic gaits at different speeds during transfer}

%%%%%%%%%%%%%%%%%%%%%%%%%%%%%%%%%%%%%%%%%%%%
% Figure 4: Maintain velocity task
% A: Comparison of actual speed vs target speed
% B: Illustration of the task
% C: Swing stance plots
%%%%%%%%%%%%%%%%%%%%%%%%%%%%%%%%%%%%%%%%%%%%
\begin{figure}[!htbp]
  \centering
  \includegraphics[width=\textwidth]{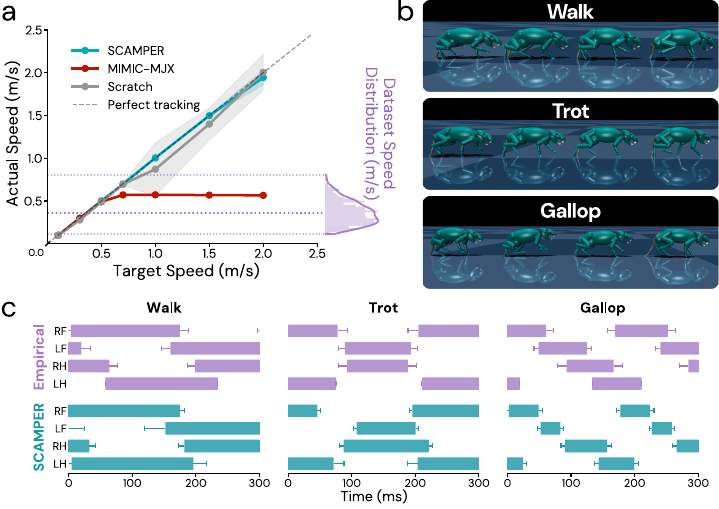}
  \caption{\textbf{SCAMPER produces realistic gaits at varying speeds.} \textbf{(a)} Actual vs.\ target speed for SCAMPER (\textit{blue}), MIMIC-MJX (\textit{red}), and the from-scratch model (\textit{gray}); lines show means, shaded bands show $\pm1$\,s.d. across seeds, dashed line indicates perfect tracking. The training-data speed distribution is shown in purple on the right. \textbf{(b)} Renderings of SCAMPER walking, trotting, and galloping at increasing target speeds. \textbf{(c)} Stance-phase plots for each limb (RF, LF, RH, LH) across walk, trot, and gallop gaits, comparing empirical rat data (\textit{purple}; data from \cite{danner_spinal_2023}) to SCAMPER (\textit{blue}); bars show stance duration, error bars show variability across strides.}
  \label{fig:maintain-vel}
\end{figure}
%%%%%%%%%%%%%%%%%%%%%%%%%%%%%%%%%%%%%%%%%%%%

Having confirmed that the behavior prior enables diverse and realistic random generation, we next tested whether residual control of the motion prior enables effective transfer to RL tasks. We began with a dense-reward locomotion task: at every timestep, the agent received a reward that decreased linearly with the difference between the target and actual speed. We trained separate policy networks for each target speed. SCAMPER closely matched the target speed across a range of realistic rat locomotion speeds (Fig.~\ref{fig:maintain-vel}a, \textit{blue line}) whereas MIMIC-MJX failed to generalize beyond speeds seen in the training data (Fig.~\ref{fig:maintain-vel}a, \textit{red line}). A model trained from scratch (as in \citet{merel_deep_2019}) also matched target speeds but with deeply unnatural motion and higher variance in its performance (Fig.~\ref{fig:maintain-vel}a, \textit{gray line}).

To better understand the resulting movement we examined the gaits at different speeds. Qualitatively, SCAMPER produced distinct gaits resembling walking, trotting, and galloping in real rats (Fig.~\ref{fig:maintain-vel}b). Notably, the training data---collected from freely behaving rats in a circular arena---contained no sustained or high-speed locomotion, and thus no trot or gallop. The emergence of these gaits in SCAMPER suggests it generalizes out-of-distribution while maintaining natural locomotion.

To verify this, we performed a quantitative gait analysis, as is standard in behavioral experiments \cite{hamers_automated_2001, mendes_quantification_2015}, comparing SCAMPER, MIMIC-MJX, and the from-scratch model against real rat data \cite{danner_spinal_2023}. For each gait (walk, trot, gallop), SCAMPER matched real rat profiles both qualitatively (Fig.~\ref{fig:maintain-vel}c) and quantitatively (Table~\ref{tab:gait-summary}). Gait profiles of MIMIC-MJX and the from-scratch model significantly diverged from the real rats (Fig.~\ref{fig:supp-gait}). Notably, MIMIC-MJX could not match target speeds above 0.5 m/s, so the higher-speed gaits (trot, gallop) never emerged. Together, these results show that residual control of the motion prior enables effective task transfer with natural movement and out-of-distribution gait generalization (see SuppVid[5-14] for videos demonstrating these results; App.~\ref{app:supvids}).

%%%%%%%%%%%%%%%%%%%%%%%%%%%%%%%%%%%%%%%%%%%%
% Table 1: Swing-stance metric comparisons for walking
%%%%%%%%%%%%%%%%%%%%%%%%%%%%%%%%%%%%%%%%%%%%
\begin{table}[t]
  \centering
  \caption{Per-metric distance to empirical rat data, computed as the absolute difference in means divided by empirical standard deviation, $|\mu_{emp}-\mu_{virtual}|/\sigma_{\text{emp}}$ ($\downarrow$). Values shown are the average across metrics in each category: Coordination (temporal symmetry, inter-limb phase couplings), Pattern (phase couplings, duty factors), and All (18 gait metrics). (See Table~\ref{tab:gait-full} for all metrics)}
  \label{tab:gait-summary}
  \small
  \setlength{\tabcolsep}{4pt}
  \begin{tabular}{@{} l ccc ccc ccc @{}}
    \toprule
    & \multicolumn{3}{c}{Walk} & \multicolumn{3}{c}{Trot} & \multicolumn{3}{c}{Gallop} \\
    \cmidrule(lr){2-4} \cmidrule(lr){5-7} \cmidrule(lr){8-10}
    Method & Coord. & Pattern & All & Coord. & Pattern & All & Coord. & Pattern & All \\
    \midrule
    SCAMPER   & \textbf{1.49} & \textbf{0.95} & \textbf{0.59} & \textbf{1.27} & \textbf{1.40} & \textbf{1.19} & \textbf{1.35} & \textbf{2.29} & \textbf{1.91} \\
    MIMIC-MJX & 1.83 & 1.68 & 1.28 & 5.15 & 3.62 & 2.72 & 1.96 & 2.51 & 2.74 \\
    Scratch   & 4.83 & 6.55 & 3.44 & 7.84 & 14.69 & 11.01 & 5.93 & 8.28 & 5.91 \\
    \bottomrule
  \end{tabular}
\end{table}

%%%%%%%%%%%%%%%%%%%%%%%%%%%%%%%%%%%%%%%%%%%%%%%%%%%%%%%%%%%%%%%%%%%%%%%%%%%%%%%%%%%%%%%
\FloatBarrier
\subsection{The motion prior expands the stability and precision of rearing behaviors}

%%%%%%%%%%%%%%%%%%%%%%%%%%%%%%%%%%%%%%%%%%%%
% Figure 5: Walk-rear task 
% A: Illustration of the task
% B: Success rate curves of SCAMPER vs MIMIC-MJX vs Scratch throughout training
%%%%%%%%%%%%%%%%%%%%%%%%%%%%%%%%%%%%%%%%%%%%
\begin{figure}[!htbp]
  \centering
  \includegraphics[width=\textwidth]{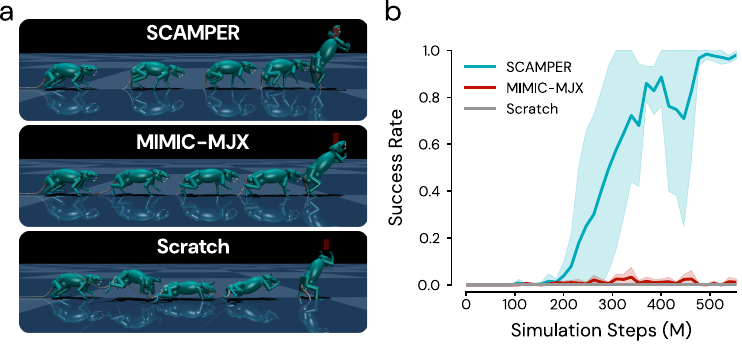}
  \caption{\textbf{SCAMPER expands the precision and stability of learnable rearing behaviors.} \textbf{(a)} Example rollouts of a walk-then-rear policy, with the target rearing height indicated by the red box, for policies trained using SCAMPER (\textit{top}), MIMIC-MJX (\textit{middle}), and trained from scratch (\textit{bottom}). \textbf{(b)} Task success rate over simulation steps for each method; lines show the mean across seeds and bands show $\pm1$\,s.d. Only the SCAMPER policy solves the task consistently within the training budget. }
  \label{fig:walk-rear}
\end{figure}
%%%%%%%%%%%%%%%%%%%%%%%%%%%%%%%%%%%%%%%%%%%%

To verify that these results extend beyond locomotion, we tested a dense-reward rearing task in which the virtual rat must lift its head to a target location (Fig.~\ref{fig:walk-rear}a). The reward decreases linearly with the between the skull and the target, and a one-time bonus reward is obtained for holding at the target for a full second. We set the target at the maximum rearing height observed in the training data to test each model's ability to use the full range of motion of the virtual body.

SCAMPER eventually achieved a high success rate (Fig.~\ref{fig:walk-rear}b, \textit{blue line}). In contrast, MIMIC-MJX rarely held the target height long enough (Fig.~\ref{fig:walk-rear}b, \textit{red line}), and the from-scratch model never reached the target, unable to use the full dynamic range of the rat body in a stable manner (Fig.~\ref{fig:walk-rear}b, \textit{gray line}). Thus, SCAMPER enables effective task transfer with realistic motion for both locomotion and rearing (see SuppVid[15-17] for videos demonstrating these results; App.~\ref{app:supvids}).

%%%%%%%%%%%%%%%%%%%%%%%%%%%%%%%%%%%%%%%%%%%%%%%%%%%%%%%%%%%%%%%%%%%%%%%%%%%%%%%%%%%%%%%
\subsection{Adaptive prior gating balances exploration and exploitation in a sparse-reward foraging task}\label{sec:forage}

Finally, we tested whether SCAMPER's gating mechanism could switch between the behavior and motion priors to balance exploration and exploitation in a sparse-reward task. We chose a foraging task relevant to behavioral neuroscience: the virtual rodent must explore a circular arena and navigate to goal locations to collect treats (Fig.~\ref{fig:arena-forage}a). In each trial, the rodent spawned at the center of the arena. To force exploration, treats spawned only outside a minimum radius. Each treat emitted a Gaussian ``scent'' field, with the rodent receiving the local odor gradient as a state variable. Collecting a treat (head overlapping its location) yielded reward and removed both the treat and its scent.

We compared the full SCAMPER model (Fig.~\ref{fig:arena-forage}b, \textit{blue line}) against a motion-prior-only variant (Fig.~\ref{fig:arena-forage}b, \textit{purple line}), MIMIC-MJX (Fig.~\ref{fig:arena-forage}b, \textit{red line}), and the from-scratch model (Fig.~\ref{fig:arena-forage}b, \textit{gray line}). All models found treats at small spawn radii, where little exploration was required. However, only the full SCAMPER model with the ability to switch to random generation via the behavior prior (Fig.~\ref{fig:supp-foragetraj}) was capable of solving the task when the minimum spawn radius was large and extensive exploration was necessary. This demonstrates that combining the behavior and motion priors enables SCAMPER to train on realistic sparse-reward tasks that require effective exploration (see SuppVid[18-22] for videos demonstrating these results; App.~\ref{app:supvids}).

%%%%%%%%%%%%%%%%%%%%%%%%%%%%%%%%%%%%%%%%%%%%
% Figure 6: Arena forage task
% A: Illustration of the RL task (2D Heatmap)
% B: Mean collected treats vs minimum spawning radius for Gated Priors, Motion Prior, MIMIC-MJX and Scratch
%%%%%%%%%%%%%%%%%%%%%%%%%%%%%%%%%%%%%%%%%%%%
\begin{figure}[!htbp]
  \centering
  \includegraphics[width=0.9\textwidth]{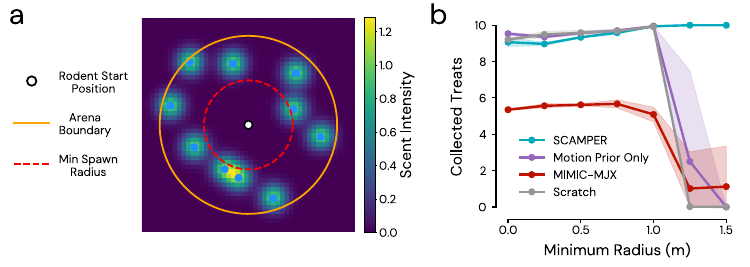}
  \caption{\textbf{SCAMPER enables efficient foraging under increasingly sparse reward.} \textbf{(a)} Arena forage task: treats (\textit{blue circles}) are spawned uniformly within an annulus bounded by the minimum spawn radius (\textit{red dashed}) and the arena boundary (\textit{orange}), with a scent field (colorbar) guiding the rodent from its start position (\textit{white}). \textbf{(b)} Number of treats collected per episode as a function of the minimum spawn radius, for policies trained using SCAMPER (\textit{blue}), a motion-prior-only variant (\textit{purple}), MIMIC-MJX (\textit{red}), and trained from scratch (\textit{gray}). Lines show the mean across seeds and bands show $\pm1$\,s.d.}
  \label{fig:arena-forage}
\end{figure}
%%%%%%%%%%%%%%%%%%%%%%%%%%%%%%%%%%%%%%%%%%%%

\FloatBarrier
%%%%%%%%%%%%%%%%%%%%%%%%%%%%%%%%%%%%%%%%%%%%%%%%%%%%%%%%%%%%%%%%%%%%%%%%%%%%%%%%%%%%%%%
%%%%%%%%%%%%%%%%%%%%%%%%%%%%%%%%%%%%%%%%%%%%%%%%%%%%%%%%%%%%%%%%%%%%%%%%%%%%%%%%%%%%%%%
\section{Discussion}
\textbf{Summary.} In this paper, we introduced SCAMPER, a system for training a realistic rat body on realistic behavioral tasks. SCAMPER enables natural movement generation during both goal-directed behavior and random exploration, thanks to its dual-prior architecture. We demonstrated that the behavior prior expands the diversity of randomly sampled motion, while preserving motion realism. We also showed that through residual control of the motion prior, SCAMPER can produce realistic gaits at different speeds and fully leverage its rearing capabilities. Furthermore, we demonstrated that by combining both priors the virtual rodent can solve more difficult sparse-reward tasks that demand exploration.

\textbf{Significance for neuro-AI.} SCAMPER is designed primarily as a tool for neuro-AI. It builds on established techniques for imitation, latent skill learning, and residual control; its additions, a behavior prior distilled from the frozen encoder and a learned gate between the two priors, extend naturalistic motion from goal-directed control to exploration. Naturalistic exploration matters for neuroscience: most experimental tasks offer only sparse reward, so models must explore to learn them, and unnatural movement would confound comparisons with neural data. \citet{aldarondo_virtual_2024} found that the internal representations of a purely imitative virtual rat matched neural activity in rat motor cortex and dorsolateral striatum. SCAMPER makes it possible to extend such comparisons to goal-directed tasks, with models and animals trained on the same task and compared throughout learning.

%%%%%%%%%%%%%%%%%%%%%%%%%%%%%%%%%%%%%%%%%%%%%%%%%%%%%%%%%%%%%%%%%%%%%%%%%%%%%%%%%%%%%%%
\textbf{Limitations \& next steps.} There are several limitations to this work that should be recognized. First and foremost, the dataset that we used for training the model was very limited. It consisted of only 1 hour of a rodent freely behaving in an open arena. As such, it contained only simple behaviors, namely walking, rearing, sitting still, and grooming. A more extensive dataset with rats engaging in more diverse behaviors, exhibiting longer bouts of consistent locomotion at various speeds, and movement in more diverse landscapes, would be of great benefit. Another key, but missing, component in the system we presented here is the presence of rich, multi-modal sensory inputs. The current version of SCAMPER uses only proprioceptive inputs, and a simple proxy for scent, effectively making the rat blind and deaf. As the field moves towards more sophisticated embodied models and experiments, this limitation will become more pressing to address. Future work can expand SCAMPER with a larger array of realistic sensory inputs. Finally, one key difference between SCAMPER and real motion control in animals is the use of torque actuators rather than muscles. Real animals must learn to control their bodies through activation of muscle fibers, and a good embodied model would likely take this facet of animal behavior into consideration. In principle, SCAMPER is agnostic to the actuator type, so as new virtual bodies get developed, SCAMPER could be applied to muscle-based models.

\newpage

%%%%%%%%%%%%%%%%%%%%%%%%%%%%%%%%%%%%%%%%%%%%%%%%%%%%%%%%%%%%%%%%%%%%%%%%%%%%%%%%%%%%%%%
%%%%%%%%%%%%%%%%%%%%%%%%%%%%%%%%%%%%%%%%%%%%%%%%%%%%%%%%%%%%%%%%%%%%%%%%%%%%%%%%%%%%%%%
\begin{ack}
The authors would like to thank Nanda Krishna, Jo\~{a}o Sacramento and Guillaume Lajoie for discussions, support, and feedback. AS was supported by a Fonds de recherche du Qu\'ebec -- Nature et technologies (FRQNT) Master's Research Scholarship, a Natural Sciences and Engineering Research Council of Canada (NSERC) Canada Graduate Research Scholarship -- Master's (CGRS~M), and a UNIQUE Excellence Scholarship. CYZ was supported by the NSF CSGrad4US Fellowship. This work was also supported by a L.I.F.E. Foundation grant to TDP, a Salk Innovation Award to TDP, an NIH BRAIN Initiative grant (U01NS136507) to BP\"O and TDP, a SFARI Investigator Award to BP\"O, and an ARNI award to BP\"O and BAR. This research was enabled in part by compute resources provided by Mila (\url{https://mila.quebec}), Calcul Qu\'ebec (\url{https://www.calculquebec.ca}), and the Digital Research Alliance of Canada (\url{https://alliancecan.ca}).
\end{ack}

%%%%%%%%%%%%%%%%%%%%%%%%%%%%%%%%%%%%%%%%%%%%%%%%%%%%%%%%%%%%%%%%%%%%%%%%%%%%%%%%%%%%%%%
%%%%%%%%%%%%%%%%%%%%%%%%%%%%%%%%%%%%%%%%%%%%%%%%%%%%%%%%%%%%%%%%%%%%%%%%%%%%%%%%%%%%%%%
\section*{Competing interests}
BAR is a Research Scientist at Google and AS is a Student Researcher at Google. The authors declare no other competing interests.

\bibliographystyle{unsrtnat}
\bibliography{references}

\newpage
\appendix

%%%%%%%%%%%%%%%%%%%%%%%%%%%%%%%%%%%%%%%%%%%%%%%%%%%%%%%%%%%%%%%%%%%%%%%%%%%%%%%%%%%%%%%
%%%%%%%%%%%%%%%%%%%%%%%%%%%%%%%%%%%%%%%%%%%%%%%%%%%%%%%%%%%%%%%%%%%%%%%%%%%%%%%%%%%%%%%
\section{Proofs}
\label{app:proofs}

We restate each proposition from Section~\ref{sec:priors_analysis} and provide the full proof.

\setcounter{proposition}{0}

\begin{proposition}[Information bottleneck under joint training]
When the encoder $q_\phi(\mathbf{z}|\mathbf{g},\mathbf{x})$ and the state-dependent prior $p_\psi(\mathbf{z}|\mathbf{x})$ are trained jointly, minimizing their expected KL divergence decomposes as
\begin{equation}
    \min_{\phi, \psi} \mathbb{E}_{p(\mathbf{x},\mathbf{g})}\left[D_{\mathrm{KL}}(q_\phi \parallel p_\psi)\right] = \min_{\phi, \psi} \left[ I_\phi(Z; G \mid X) + \mathbb{E}_{p(\mathbf{x})} \left[ D_{\mathrm{KL}}(m_\phi \parallel p_\psi) \right] \right],
\end{equation}
where $m_\phi(\mathbf{z}|\mathbf{x})=\mathbb{E}_{p(\mathbf{g}|\mathbf{x})}[q_\phi(\mathbf{z}|\mathbf{g},\mathbf{x})]$ is the marginal posterior induced by the encoder, which evolved during joint training.
\end{proposition}

\begin{proof}
Let the joint optimization objective be the expected KL divergence over the data distribution:
\begin{equation}
    \min_{\phi, \psi} \mathcal{L}(\phi, \psi) = \min_{\phi, \psi} \mathbb{E}_{p(\mathbf{x},\mathbf{g})} \left[ D_{\mathrm{KL}}\big( q_\phi(\mathbf{z} | \mathbf{g}, \mathbf{x}) \parallel p_\psi(\mathbf{z} | \mathbf{x}) \big) \right].
\end{equation}
Expanding the expected KL divergence into its integral form:
\begin{equation}
    \min_{\phi,\psi}\mathcal{L}(\phi, \psi) = \min_{\phi,\psi}\int p(\mathbf{x}) \int p(\mathbf{g}|\mathbf{x}) \int q_\phi(\mathbf{z}|\mathbf{g}, \mathbf{x}) \log \frac{q_\phi(\mathbf{z}|\mathbf{g}, \mathbf{x})}{p_\psi(\mathbf{z}|\mathbf{x})} \, d\mathbf{z} \, d\mathbf{g} \, d\mathbf{x}.
\end{equation}
Define the marginal posterior (which depends on the encoder parameters $\phi$):
\begin{equation}
    m_\phi(\mathbf{z}|\mathbf{x}) \triangleq \int q_\phi(\mathbf{z}|\mathbf{g}, \mathbf{x}) p(\mathbf{g}|\mathbf{x}) \, d\mathbf{g}.
\end{equation}
Multiplying numerator and denominator inside the logarithm by $m_\phi(\mathbf{z}|\mathbf{x})$:
\begin{equation}
    \log \frac{q_\phi(\mathbf{z}|\mathbf{g}, \mathbf{x})}{p_\psi(\mathbf{z}|\mathbf{x})} = \log \frac{q_\phi(\mathbf{z}|\mathbf{g}, \mathbf{x})}{p_\psi(\mathbf{z}|\mathbf{x})}\frac{m_\phi(\mathbf{z}|\mathbf{x})}{m_\phi(\mathbf{z}|\mathbf{x})}= \log \frac{q_\phi(\mathbf{z}|\mathbf{g}, \mathbf{x})}{m_\phi(\mathbf{z}|\mathbf{x})} + \log \frac{m_\phi(\mathbf{z}|\mathbf{x})}{p_\psi(\mathbf{z}|\mathbf{x})}.
\end{equation}
Substituting this decomposition back into the integral yields two terms:
\begin{equation}
\begin{aligned}
    \min_{\phi,\psi}\mathcal{L}(\phi, \psi) &= \min_{\phi,\psi} \left[ \underbrace{ \int p(\mathbf{x}) \int p(\mathbf{g}|\mathbf{x}) \int q_\phi(\mathbf{z}|\mathbf{g}, \mathbf{x}) \log \frac{q_\phi(\mathbf{z}|\mathbf{g}, \mathbf{x})}{m_\phi(\mathbf{z}|\mathbf{x})} \, d\mathbf{z} \, d\mathbf{g} \, d\mathbf{x} }_{\text{Term 1}}\right. \\
    &\left. \quad + \underbrace{ \int p(\mathbf{x}) \int p(\mathbf{g}|\mathbf{x}) \int q_\phi(\mathbf{z}|\mathbf{g}, \mathbf{x}) \log \frac{m_\phi(\mathbf{z}|\mathbf{x})}{p_\psi(\mathbf{z}|\mathbf{x})} \, d\mathbf{z} \, d\mathbf{g} \, d\mathbf{x} }_{\text{Term 2}}\right].
\end{aligned}
\end{equation}

\textbf{Term 1} is, by definition, the conditional mutual information between the latent $\mathbf{z}$ and the goal $\mathbf{g}$ given the state $\mathbf{x}$:
\begin{equation}
    \text{Term 1} = I_\phi(Z; G \mid X).
\end{equation}

\textbf{Term 2}: Since $\log \frac{m_\phi(\mathbf{z}|\mathbf{x})}{p_\psi(\mathbf{z}|\mathbf{x})}$ is independent of $\mathbf{g}$, and the integrand is absolutely integrable\footnote{Absolute integrability follows from the assumption that $D_\mathrm{KL}(q_\phi||p_\psi)<\infty$ on the support of $p(\mathbf{x},\mathbf{g})$.}, we apply Fubini's theorem to evaluate the integral over $\mathbf{g}$ first:
\begin{equation}
\begin{aligned}
    \text{Term 2} &= \int p(\mathbf{x}) \int \underbrace{\left( \int q_\phi(\mathbf{z}|\mathbf{g}, \mathbf{x}) p(\mathbf{g}|\mathbf{x}) \, d\mathbf{g} \right)}_{m_\phi(\mathbf{z}|\mathbf{x})} \log \frac{m_\phi(\mathbf{z}|\mathbf{x})}{p_\psi(\mathbf{z}|\mathbf{x})} \, d\mathbf{z} \, d\mathbf{x} \\
    &= \mathbb{E}_{p(\mathbf{x})} \left[ D_{\mathrm{KL}}\big( m_\phi(\mathbf{z}|\mathbf{x}) \parallel p_\psi(\mathbf{z}|\mathbf{x}) \big) \right].
\end{aligned}
\end{equation}

Combining the two terms gives the desired decomposition:
\begin{equation}
    \min_{\phi, \psi} \mathcal{L}(\phi, \psi) = \min_{\phi, \psi} \left[ I_\phi(Z; G \mid X) + \mathbb{E}_{p(\mathbf{x})} \left[ D_{\mathrm{KL}}\big( m_\phi(\mathbf{z}|\mathbf{x}) \parallel p_\psi(\mathbf{z}|\mathbf{x}) \big) \right] \right].
\end{equation}
\end{proof}

\begin{proposition}[Marginal posterior approximation under frozen encoder]
When the encoder $q_\phi(\mathbf{z}|\mathbf{g},\mathbf{x})$ is frozen, minimizing the expected KL divergence between the encoder and a state-dependent prior $p_\varphi(\mathbf{z}|\mathbf{x})$ is equivalent to
\begin{equation}
    \min_\varphi \mathbb{E}_{p(\mathbf{x},\mathbf{g})}\left[D_{\mathrm{KL}}(q_\phi \parallel p_\varphi)\right] = I_\phi(Z; G \mid X)+ \min_\varphi \mathbb{E}_{p(\mathbf{x})}\left[D_{\mathrm{KL}}(m_\phi \parallel p_\varphi)\right],
\end{equation}
where $m_\phi(\mathbf{z}|\mathbf{x}) = \mathbb{E}_{p(\mathbf{g}|\mathbf{x})}[q_\phi(\mathbf{z}|\mathbf{g},\mathbf{x})]$ is the marginal posterior induced by the frozen encoder, held fixed during optimization. The conditional mutual information, $I_\phi(Z; G \mid X)$, does not depend on the optimization variable $\varphi$; it is therefore an additive constant that sits outside the minimization. Hence, the optimal behavior prior is the closest state-conditional approximation to the marginal posterior $m_\phi$.
\end{proposition}

\begin{proof}
Let the distillation objective be the expected KL divergence over the data distribution:
\begin{equation}
    \min_{\varphi}\mathcal{L}(\varphi) = \min_\varphi\mathbb{E}_{p(\mathbf{x},\mathbf{g})} \left[ D_{\mathrm{KL}}\big( q_\phi(\mathbf{z} | \mathbf{g}, \mathbf{x}) \parallel p_\varphi(\mathbf{z} | \mathbf{x}) \big) \right].
\end{equation}
Expanding the inner KL divergence as the difference between cross-entropy and entropy:
\begin{equation}
    D_{\mathrm{KL}}(q_\phi \parallel p_\varphi) = \underbrace{- \int q_\phi(\mathbf{z}|\mathbf{g}, \mathbf{x}) \log p_\varphi(\mathbf{z}|\mathbf{x}) \, d\mathbf{z}}_{H(q_\phi, p_\varphi)} + \underbrace{\int q_\phi(\mathbf{z}|\mathbf{g}, \mathbf{x}) \log q_\phi(\mathbf{z}|\mathbf{g}, \mathbf{x}) \, d\mathbf{z}}_{-H(q_\phi)}.
\end{equation}

Since the encoder parameters $\phi$ are frozen, the entropy $H(q_\phi)$ is constant with respect to $\varphi$ and can be pulled outside the minimization:
\begin{equation}\label{eqn:entropy}
    \min_\varphi\mathcal{L}(\varphi) = \min_\varphi \left[ \mathbb{E}_{p(\mathbf{x},\mathbf{g})}[H(q_\phi,p_\varphi)] \right]-\mathbb{E}_{p(\mathbf{x},\mathbf{g})}[H(q_\phi)].
\end{equation}
First we focus on the first term inside the minimization and expand it out:
\begin{equation}
    \mathbb{E}_{p(\mathbf{x},\mathbf{g})}[H(q_\phi,p_\varphi)] = - \int p(\mathbf{x}) \int p(\mathbf{g}|\mathbf{x}) \int q_\phi(\mathbf{z} | \mathbf{g}, \mathbf{x}) \log p_\varphi(\mathbf{z} | \mathbf{x}) \, d\mathbf{z} \, d\mathbf{g} \, d\mathbf{x}.
\end{equation}
Assuming absolute integrability, we apply Fubini's theorem to swap the order of integration between $\mathbf{g}$ and $\mathbf{z}$, and factor out $\log p_\varphi(\mathbf{z}|\mathbf{x})$ since it is independent of $\mathbf{g}$:
\begin{equation}
\begin{aligned}
    \mathbb{E}_{p(\mathbf{x},\mathbf{g})}[H(q_\phi,p_\varphi)] &= - \int p(\mathbf{x}) \int \log p_\varphi(\mathbf{z} | \mathbf{x}) \underbrace{\left( \int q_\phi(\mathbf{z} | \mathbf{g}, \mathbf{x}) p(\mathbf{g}|\mathbf{x}) \, d\mathbf{g} \right)}_{m_\phi(\mathbf{z}|\mathbf{x})} \, d\mathbf{z} \, d\mathbf{x} \\
    &= - \int p(\mathbf{x}) \int m_\phi(\mathbf{z}|\mathbf{x}) \log p_\varphi(\mathbf{z} | \mathbf{x}) \, d\mathbf{z} \, d\mathbf{x}\\
    &= \mathbb{E}_{p(\mathbf{x})}[H(m_\phi,p_\varphi)].
\end{aligned}
\end{equation}
This is the expectation of the cross-entropy between $m_\phi$ and $p_\varphi$. Substituting back into equation (\ref{eqn:entropy}) yields:
\begin{equation}
    \min_\varphi\mathcal{L}(\varphi) = \min_\varphi \left[ \mathbb{E}_{p(\mathbf{x})}[H(m_\phi,p_\varphi)] \right]-\mathbb{E}_{p(\mathbf{x},\mathbf{g})}[H(q_\phi)].
\end{equation}

To relate it to the KL divergence, we add and subtract the expectation of the entropy of the marginal posterior:
\begin{equation}
    \min_\varphi\mathcal{L}(\varphi) = \min_\varphi \left[ \mathbb{E}_{p(\mathbf{x})}[H(m_\phi,p_\varphi)] \right]-\mathbb{E}_{p(\mathbf{x})}[H(m_\phi)]+\mathbb{E}_{p(\mathbf{x})}[H(m_\phi)]-\mathbb{E}_{p(\mathbf{x},\mathbf{g})}[H(q_\phi)].
\end{equation}
Pulling $\mathbb{E}_{p(\mathbf{x})}[H(m_\phi)]$ into the minimization operation and using the identity $H(m_\phi,p_\varphi)=D_{\mathrm{KL}}(m_\phi||p_\varphi)+H(m_\phi)$:
\begin{equation}
    \min_\varphi\mathcal{L}(\varphi) = \min_\varphi \mathbb{E}_{p(\mathbf{x})} \left[ D_{\mathrm{KL}}(m_\phi(\mathbf{z}|\mathbf{x})\parallel p_\varphi(\mathbf{z}|\mathbf{x})) \right]+\Big( \mathbb{E}_{p(\mathbf{x})}[H(m_\phi)]-\mathbb{E}_{p(\mathbf{x},\mathbf{g})}[H(q_\phi)] \Big).
\end{equation}
The bracketed difference of entropies is, by definition, the conditional mutual information between the latent and the goal given the state under the frozen encoder, therefore:
\begin{equation}
    \min_\varphi\mathcal{L}(\varphi) = \min_\varphi \mathbb{E}_{p(\mathbf{x})} \left[ D_{\mathrm{KL}}(m_\phi(\mathbf{z}|\mathbf{x})\parallel p_\varphi(\mathbf{z}|\mathbf{x})) \right] + I_\phi(Z;G \mid X).
\end{equation}
\end{proof}

%%%%%%%%%%%%%%%%%%%%%%%%%%%%%%%%%%%%%%%%%%%%%%%%%%%%%%%%%%%%%%%%%%%%%%%%%%%%%%%%%%%%%%%
%%%%%%%%%%%%%%%%%%%%%%%%%%%%%%%%%%%%%%%%%%%%%%%%%%%%%%%%%%%%%%%%%%%%%%%%%%%%%%%%%%%%%%%
\section{Dataset}
\textbf{Motion capture.} We use a rat motion dataset collected by \citet{aldarondo_virtual_2024}. The dataset is composed of 842 5\,s clips of rats freely behaving in a circular enclosure. The mocap data is characterized by 23 keypoints captured at 50\,Hz using DANNCE v1.3 \cite{Dunn2021-zc} from multicamera video. The dataset contains various kinds of walking, rearing, grooming and prone behaviors. The mocap keypoints are aligned to the virtual rodent body model through STAC-MJX \cite{zhang_mimic-mjx_2025}.

\textbf{Empirical rat locomotion data.} The empirical data used in the gait analysis (phase differences, duty factors, and gait classifications) were obtained from \citet{danner_spinal_2023}, publicly available at \url{https://github.com/dannerlab/rat-sci-locomotion} under a Creative Commons Attribution--ShareAlike 4.0 International license (CC~BY--SA~4.0, \url{https://creativecommons.org/licenses/by-sa/4.0/}). The data was used unmodified.

%%%%%%%%%%%%%%%%%%%%%%%%%%%%%%%%%%%%%%%%%%%%%%%%%%%%%%%%%%%%%%%%%%%%%%%%%%%%%%%%%%%%%%%
%%%%%%%%%%%%%%%%%%%%%%%%%%%%%%%%%%%%%%%%%%%%%%%%%%%%%%%%%%%%%%%%%%%%%%%%%%%%%%%%%%%%%%%
\section{Supplementary videos}
\label{app:supvids}
We host the supplementary videos demonstrating the rodent performing the various tasks described in this work on the project's page (\href{https://aidansirbu.com/SCAMPER-site/}{\texttt{https://aidansirbu.com/SCAMPER-site/}}). Specific videos are also referenced throughout the text: ``SuppVid<x>'' refers to the corresponding video on the website and the .mp4 file in the supplementary video package.

%%%%%%%%%%%%%%%%%%%%%%%%%%%%%%%%%%%%%%%%%%%%%%%%%%%%%%%%%%%%%%%%%%%%%%%%%%%%%%%%%%%%%%%
%%%%%%%%%%%%%%%%%%%%%%%%%%%%%%%%%%%%%%%%%%%%%%%%%%%%%%%%%%%%%%%%%%%%%%%%%%%%%%%%%%%%%%%
\section{Implementation \& compute}
\textbf{Implementation.} SCAMPER is implemented on top of the MuJoCo physics engine \cite{Todorov2012-bp}. We accelerate simulation using MuJoCo XLA (MJX) \cite{MuJoCo-XLA-AuthorsUnknown-il}, a JAX-based \cite{Bradbury2018-ge} reimplementation of MuJoCo, which enables GPU-vectorized physics rollouts and domain randomization to be applied directly to the per-environment parameters. Environments are built on top of MuJoCo Playground's \cite{Zakka2025-wi} base class and wrappers, which expose MJX simulation through a JAX-compatible interface and adapt our environments to the Brax \cite{Freeman2021-fb} training API. RL is performed with Brax, whose PPO implementation and distributed training utilities we extend to support our networks and distillation objectives. Pretraining (Stages 1-3) was performed using the JAX backend of MJX and task transfer was performed using the Warp backend of MJX. MIMIC-MJX baselines were implemented using the training setup and hyperparameters described in \citet{zhang_mimic-mjx_2025}. The from-scratch baseline uses position (PD) actuators, whereas all other models use torque actuators. Torque actuators are highly unstable and difficult to learn from scratch, so we give the from-scratch model PD actuators to give it a fair chance to compete.

\textbf{Compute.} All experiments were run on a single GPU with $4$ CPU cores and $48$\,GB of host RAM. Pretraining used NVIDIA L40S (48\,GB) GPUs: the Stage~1 expert imitator trains for $3\times 10^{9}$ environment steps over $4096$ parallel envs in $\approx 14$\,h, the Stage~2 distillation of the encoder/decoder and motion prior trains for $1\times 10^{9}$ steps over $4096$ envs in $\approx 12$\,h, and the Stage~3 behavior prior trains for $2\times 10^{9}$ steps over $4096$ envs in $\approx 30$\,h, for a total of roughly $56$ L40S-hours per end-to-end pretraining run. The three task-transfer benchmarks were trained with PPO under MuJoCo Warp on NVIDIA A100 (80\,GB) GPUs, with four seeds per condition. Walk-then-rear was run for the SCAMPER, MIMIC-MJX, and from-scratch models, with each job training for $1\times 10^{9}$ steps over $512$ envs in $\approx 11$\,h ($12$ jobs total, $\approx 132$ A100-hours). Locomotion was run for the same three models swept over $7$ target speeds, with each job training for $4\times 10^{8}$ steps over $1024$ envs in $\approx 10$\,h ($84$ jobs total, $\approx 840$ A100-hours). Arena foraging was run for SCAMPER, a motion-prior-only variant, MIMIC-MJX, and the from-scratch model swept over $7$ minimum spawn radii, with each job training for $2\times 10^{9}$ steps over $1024$ envs in $\approx 6$\,h ($112$ jobs total, $\approx 672$ A100-hours). One full reproduction of the paper's training pipeline therefore costs roughly $56$ L40S-GPU-hours and $1{,}650$ A100-GPU-hours.

%%%%%%%%%%%%%%%%%%%%%%%%%%%%%%%%%%%%%%%%%%%%%%%%%%%%%%%%%%%%%%%%%%%%%%%%%%%%%%%%%%%%%%%
%%%%%%%%%%%%%%%%%%%%%%%%%%%%%%%%%%%%%%%%%%%%%%%%%%%%%%%%%%%%%%%%%%%%%%%%%%%%%%%%%%%%%%%
\section{Tasks}
\label{app:tasks}

%%%%%%%%%%%%%%%%%%%%%%%%%%%%%%%%%%%%%%%%%%%%%%%%%%%%%%%%%%%%%%%%%%%%%%%%%%%%%%%%%%%%%%%
\subsection{Stage 1: Expert imitator}
\label{app:imitation_mlp}

\textbf{Task.} The expert imitator $\pi_I(\mathbf{a}_t|\mathbf{x}_t,\mathbf{g}_t)$ tracks 3D motion-capture trajectories of freely behaving rats using the biomechanical rat model of \citet{merel_deep_2019}. We use a $\sim$1 h subset of the open-source rat mocap dataset released by \citet{aldarondo_virtual_2024} (842 clips of 5\,s mocap), registered to the virtual rat via STAC-MJX \cite{zhang_mimic-mjx_2025}. At the start of each episode we sample a clip uniformly from the training partition ($80\%$ of the dataset; the remaining $20\%$ is held out for evaluation rollouts) and a starting frame uniformly in $[0, 50]$, with zero initial joint velocity. Episodes last up to $250$ reference frames. At each control step the agent observes the proprioception vector $\mathbf{x}_t$ and the egocentric goal $\mathbf{g}_t$ defined in Section~\ref{sec:preliminaries}, where $\mathbf{g}_t$ is computed from a $5$-frame reference horizon, and emits torques $\mathbf{a}_t\in\mathbb{R}^{38}$ applied to the rat's actuated joints. Control runs at $100$\,Hz on top of a $500$\,Hz MuJoCo MJX physics simulation, with a $50$\,Hz mocap reference. The body model uses torque actuators and a $0.9\times$ rescaling factor.

\textbf{Reward.} Each imitation tracking term takes the form $r_i = w_i\exp\!\big(-\tfrac{1}{2}(e_i/\sigma_i)^2\big)$, a Gaussian kernel of the per-term error $e_i$ with characteristic scale $\sigma_i$; values of $(w_i,\sigma_i)$ are tabulated in Table~\ref{tab:imitation_mlp_hparams}. The four active tracking terms penalize errors in root position, root orientation (quaternion geodesic), joint angles, and egocentric Cartesian end-effector positions. Three regularization costs penalize quadratic control magnitude, quadratic control differences (jerk), and clipped torque--velocity product, and a torso-$z$-range term adds an in-range bonus that rewards keeping the torso height within the healthy range $[0.0325, 0.5]$\,m.

\textbf{Termination.} An episode terminates early if (i) the root position deviates from the reference by more than $0.1$\,m, (ii) the root orientation deviates by more than $60^\circ$, (iii) the pose $L_2$ error exceeds $4.5$, or (iv) any state value becomes \texttt{NaN}. Episodes also terminate at the end of the reference clip.

\textbf{Domain randomization.} To encourage robust control across variations in the biomechanical model, at the start of each episode we sample physics parameters uniformly within fixed ranges and hold them fixed for the duration of the episode \cite{peng_sim--real_2018}. Randomized parameters include floor friction, static friction scale, joint armature scale, link mass scale, torso mass jitter, center-of-mass jitter, and an additive jitter on the initial joint configuration $\boldsymbol{\theta}_{0}$; full ranges are listed in Table~\ref{tab:imitation_mlp_hparams}.

\textbf{Network architecture.} The expert imitator is a multi-layer perceptron with no latent bottleneck: the concatenation of $\mathbf{x}_t$ and $\mathbf{g}_t$, normalized by their running mean and variance, is fed directly to a $[2048, 1024, 1024, 512]$ policy MLP that outputs the mean and log-standard-deviation of a tanh-squashed diagonal Gaussian over the $38$-dimensional action, so that $\mathbf{a}_t \in [-1, 1]^{38}$. The critic is a separate $[2048, 1024, 1024, 1024, 1024, 512]$ MLP estimating the state-value function. Both networks use SiLU activations with LayerNorm between hidden layers, and are trained from random initialization using LeCun initialization with zero bias.

\textbf{Training.} We train $\pi_I$ with PPO \cite{schulman_proximal_2017} for $3\times 10^{9}$ environment steps across $4{,}096$ parallel environments. Each PPO update uses a $20$-step unroll, batch size $1{,}024$, and $4$ epochs of $16$ minibatches per batch, with Adam at learning rate $10^{-4}$ and gradient clipping at $20.0$. We use discount $\gamma = 0.95$, PPO clip range $\varepsilon = 0.2$, entropy coefficient $10^{-2}$, value-loss coefficient $0.25$, and reward scaling $1$. Evaluation rollouts on the held-out clip partition are performed every $2.5\times 10^{7}$ environment steps with deterministic (mean) actions. Training is implemented in JAX \cite{Bradbury2018-ge} with MuJoCo MJX \cite{MuJoCo-XLA-AuthorsUnknown-il} and runs on a single GPU; the full list of training hyperparameters is given in Table~\ref{tab:imitation_mlp_hparams}.

%%%%%%%%%%%%%%%%%%%%%%%%%%%%%%%%%%%%%%%%%%%%%%%%%%%%%%%%%%%%%%%%%%%%%%%%%%%%%%%%%%%%%%%
\subsection{Stage 2: Online distillation}
\label{app:distill}

\textbf{Task.} Online distillation reuses the imitation environment from Stage 1 (App.~\ref{app:imitation_mlp}): identical clip set and train/eval split, control rate, physics solver, reward terms, termination criteria, and domain randomization (Table~\ref{tab:imitation_mlp_hparams}). The Stage~1 expert imitator $\pi_I$ is loaded as a frozen teacher and queried for target actions $\mathbf{a}_t$ on every visited state, while a student latent skill controller---the encoder $q_\phi(\mathbf{z}_t|\mathbf{x}_t,\mathbf{g}_t)$, decoder $p_\theta(\mathbf{a}_t|\mathbf{x}_t,\mathbf{z}_t)$, and proprioception-conditioned motion prior $p_\psi(\mathbf{z}_t|\mathbf{x}_t)$---is rolled out in simulation and trained jointly to match the teacher, using a similar DAgger \cite{ross_reduction_2011} schedule as PULSE \cite{luo_universal_2024}. Episode terminations and rewards are inherited from the underlying environment but are not used as a learning signal: the student is supervised against the teacher's actions, not the imitation reward.

\textbf{Loss.} The student minimizes the composite distillation objective
\begin{equation*}
  \mathcal{L}_{\mathrm{distill}} = \mathcal{L}_{\mathrm{DAgger}} + \lambda \mathcal{L}_{\mathrm{KL}} + \alpha \mathcal{L}_{\mathrm{AR}},
\end{equation*}
where $\mathcal{L}_{\mathrm{KL}}=D_{\mathrm{KL}}(q_\phi||p_\psi)$ regularizes the encoder toward the motion prior and $\mathcal{L}_{\mathrm{AR}}=||\boldsymbol{\mu}_t^e-\boldsymbol{\mu}_{t-1}^e||_2^2$ promotes temporal smoothness in the encoder mean. In practice $\mathcal{L}_{\mathrm{DAgger}}$ is computed on the policy's action-distribution parameters rather than on sampled actions: both the teacher MLP $\pi_I$ and the student decoder $p_\theta$ output the mean and log-standard-deviation of a tanh-squashed diagonal Gaussian over the action, and we take the MSE between these parameter vectors. This reduces to the action-matching form $\mathcal{L}_{\mathrm{DAgger}}=||\mathbf{a}_t-\mathbf{\hat{a}}_t||_2^2$ given in the main text when both policies are evaluated deterministically (i.e., at their distribution means), and is a policy-distribution distillation in the spirit of DAgger \cite{ross_reduction_2011} more generally. The KL coefficient $\lambda$ is annealed linearly from $10^{-3}$ to $10^{-4}$ over the first $25\%$ of training and held fixed thereafter; the AR(1) coefficient is fixed at $\alpha=5\times 10^{-3}$. To prevent posterior collapse and runaway sampling noise, the encoder and prior log-variances are clipped to $[-5, 0]$ (matching PULSE \cite{luo_universal_2024}).

\textbf{Network architecture.} All three networks are MLPs with SiLU activations and LayerNorm between hidden layers (matching the Stage~1 policy). The encoder maps the concatenation of $\mathbf{x}_t$ and $\mathbf{g}_t$ through hidden layers $[2048, 1024, 1024]$ to per-dimension means and log-variances of a $16$-dimensional Gaussian latent $\mathbf{z}_t$; we sample $\mathbf{z}_t$ via the standard reparameterization trick. The decoder maps the concatenation of $\mathbf{z}_t$ and $\mathbf{x}_t$ through $[1024, 1024, 512]$ to an action prediction $\hat{\mathbf{a}}_t$. The proprioception-conditioned motion prior maps $\mathbf{x}_t$ through $[1536, 1024, 512]$ to the mean and log-variance of $p_\psi(\mathbf{z}_t|\mathbf{x}_t)$. Encoder, decoder, and motion prior are randomly initialized using LeCun initialization with zero bias and trained end-to-end; the teacher $\pi_I$ remains frozen throughout.

\textbf{Training.} We collect $4{,}096$ parallel student rollouts in MuJoCo MJX with the same $20$-step unroll length used in Stage~1, train for $10^{9}$ environment steps with batch size $1{,}024$ split across $16$ minibatches and $4$ epochs per batch, and update $(\phi, \theta, \psi)$ jointly with Adam at learning rate $10^{-4}$ and global gradient-norm clipping at $10.0$. Observations are normalized by their running mean and variance, with statistics inherited from the Stage~1 normalizer to ensure consistency with the frozen teacher. Evaluation rollouts are run every $2\times 10^{7}$ environment steps with the deterministic decoder mean. The complete list of distillation hyperparameters is given in Table~\ref{tab:distill_hparams}.

%%%%%%%%%%%%%%%%%%%%%%%%%%%%%%%%%%%%%%%%%%%%%%%%%%%%%%%%%%%%%%%%%%%%%%%%%%%%%%%%%%%%%%%
\subsection{Stage 3: Behavior prior distillation}
\label{app:behavior_prior}

\textbf{Task.} Behavior prior distillation reuses the imitation environment from Stage~1 (App.~\ref{app:imitation_mlp}) for rollout collection, with domain randomization disabled. The Stage~2 latent skill controller---encoder $q_\phi(\mathbf{z}_t|\mathbf{x}_t,\mathbf{g}_t)$, decoder $p_\theta(\mathbf{a}_t|\mathbf{x}_t,\mathbf{z}_t)$, and motion prior $p_\psi(\mathbf{z}_t|\mathbf{x}_t)$---is loaded frozen and used to drive simulation: at each step we feed the encoder's mean $\boldsymbol{\mu}^e_t$ as the latent into the frozen decoder and use the resulting deterministic action mean to advance the environment. The trainable behavior prior $p_\varphi(\mathbf{z}_t|\mathbf{x}_t)$ is optimized only on the proprioceptive states visited along these trajectories; loss evaluation queries the full encoder distribution $(\boldsymbol{\mu}^e_t, \boldsymbol{\sigma}^e_t)$ at each visited state, not a sample.

\textbf{Loss.} With the encoder $q_\phi$ held fixed, the behavior prior is trained by minimizing the expected KL divergence
\begin{equation*}
  \mathcal{L}_{\mathrm{behavior}} = D_{\mathrm{KL}}\big(q_\phi(\mathbf{z}_t|\mathbf{x}_t,\mathbf{g}_t) \,\|\, p_\varphi(\mathbf{z}_t|\mathbf{x}_t)\big),
\end{equation*}
which by Proposition~\ref{prop:behavior_prior} is equivalent to matching the marginal posterior $m_\phi(\mathbf{z}_t|\mathbf{x}_t) = \mathbb{E}_{p(\mathbf{g}_t|\mathbf{x}_t)}[q_\phi(\mathbf{z}_t|\mathbf{g}_t,\mathbf{x}_t)]$. Crucially, the encoder is not modified at this stage; it provides a fixed target distribution, so the behavior prior absorbs the goal-conditional variation in $q_\phi$ rather than collapsing onto local motion continuation as the motion prior does (Section~\ref{sec:priors_analysis}).

\textbf{Network architecture.} The behavior prior is an MLP with SiLU activations and LayerNorm between hidden layers, mirroring the motion prior. It maps $\mathbf{x}_t$ through hidden sizes $[2048, 2048, 1024, 512]$ to the mean and log-variance of a $16$-dimensional Gaussian over $\mathbf{z}_t$, matching the latent dimensionality of the frozen encoder. The behavior prior is the only trainable network in this stage; encoder, decoder, and motion prior all remain frozen. The behavior prior is initialized using LeCun initialization with zero bias.

\textbf{Training.} We collect $4{,}096$ parallel rollouts in MuJoCo MJX with the same $20$-step unroll length used in earlier stages, train for $2\times 10^{9}$ environment steps with batch size $1{,}024$ split across $16$ minibatches and $4$ epochs per batch, and update $\varphi$ with Adam at learning rate $10^{-4}$ and global gradient-norm clipping at $10.0$. Observation statistics are inherited from the Stage~2 normalizer and held fixed, so that $p_\varphi$ sees the same input distribution as the frozen encoder. Evaluation rollouts are run every $2\times 10^{7}$ environment steps using the deterministic decoder mean and a deterministic behavior-prior mean (with stochastic-sampling evaluations reported separately in the experiments). The complete list of behavior-prior hyperparameters is given in Table~\ref{tab:behavior_prior_hparams}.

%%%%%%%%%%%%%%%%%%%%%%%%%%%%%%%%%%%%%%%%%%%%%%%%%%%%%%%%%%%%%%%%%%%%%%%%%%%%%%%%%%%%%%%
\subsection{Stage 4 (transfer learning): Arena foraging}
\label{app:arena_forage}

\textbf{Task.} Stage~4 fine-tunes a task-specific policy on top of the frozen Stage~3 controller. The decoder $p_\theta$, motion prior $p_\psi$, and behavior prior $p_\varphi$ are loaded frozen from the Stage~3 checkpoint (Table~\ref{tab:behavior_prior_hparams}); only the task policy $\pi_T(\boldsymbol{\delta}_t,\beta_t|\mathbf{x}_t,\mathbf{o}_t)$ and its critic are trained. In the arena foraging task, the rat is placed at the center of a flat circular arena (Fig.~\ref{fig:supp-foragerender}) and must navigate to and collect $N=10$ treats (blue spheres of radius $0.06$\,m at height $0.03$\,m). At the start of each episode the rat's initial yaw is sampled uniformly in $[0, 2\pi)$, and the $N$ treats are placed one per angular sector with sub-sector jitter, at uniformly random radii within an annulus of inner radius of variable size depending on the trial and outer radius $r_{\max}=2.0$\,m. Each treat emits an isotropic Gaussian scent field with bandwidth $\sigma=0.2$\,m. In addition to the proprioception vector $\mathbf{x}_t$ from Section~\ref{sec:preliminaries}, the agent observes a task-specific $\mathbf{o}_t$ comprising (i) the scalar total scent at the rat's skull (sum of Gaussians from currently active treats) and (ii) the egocentric gradient $\nabla\!s_t\in\mathbb{R}^3$ of the same scent field, which points toward the scent-weighted treat center.

\textbf{Reward.} The task uses a sparse reward: collecting a treat (skull entering its collision sphere) yields a large, one-time reward and the treat with its scent contribution is removed for the remainder of the episode.

\textbf{Termination.} Each episode runs for up to $2{,}000$ control steps ($20$\,s at the $100$\,Hz control rate inherited from Stage~1), or terminates early once all $N$ treats have been collected.

\textbf{Policy gating.} The task policy outputs a $16$-dimensional residual $\boldsymbol{\delta}_t$ to the motion prior's mean and a scalar gating logit $\beta_t$. The binary gate $b_t$ is thresholded with bias $0.5$ and hysteresis margin $0.2$: switching $b_t$ to $1$ requires $\beta_t > 0.7$, and switching to $0$ requires $\beta_t < 0.3$. When $b_t=1$, the latent is $\mathbf{z}_t = \boldsymbol{\mu}_t^m + \boldsymbol{\delta}_t$ (motion-prior residual control). When $b_t=0$, the latent is sampled stochastically from the behavior prior, $\mathbf{z}_t = \boldsymbol{\mu}_t^b + \epsilon\boldsymbol{\sigma}$ with $\epsilon\!\sim\!\mathcal{N}(0, 1)$ and a fixed log-variance of $-2.25$ (i.e., $\sigma=0.32$), which yields random motions that fall within the range of the natural behavioral diversity of the training data (Fig.~\ref{fig:behavior-diversity}b). The latent is passed through the frozen decoder $p_\theta$ to produce the action torques.

\textbf{Network architecture.} The task policy is an MLP with hidden sizes $[2048, 1024, 512]$ that maps the concatenation of $\mathbf{x}_t$ and $\mathbf{o}_t$ (normalized by running mean and variance) to the mean and log-standard-deviation of a tanh-squashed diagonal Gaussian over the joint $17$-dimensional output $(\boldsymbol{\delta}_t, \beta_t)$, so that $\boldsymbol{\delta}_t \in [-1, 1]^{16}$ and $\beta_t \in [-1, 1]$. The critic is a separate $[2048, 1024, 1024, 512]$ MLP estimating the state-value function. Both networks use SiLU activations with LayerNorm between hidden layers and are randomly initialized using LeCun initialization with zero bias. The decoder, motion prior, and behavior prior all remain frozen throughout this stage.

\textbf{Training.} We train $\pi_T$ with PPO \cite{schulman_proximal_2017} for $2\times 10^{9}$ environment steps across $1{,}024$ parallel environments. Each PPO update uses a $20$-step unroll, batch size $1{,}024$, and $4$ epochs of $16$ minibatches per batch, with Adam at learning rate $10^{-4}$ and gradient clipping at maximum norm $50.0$. We use discount $\gamma = 0.99$, GAE $\lambda = 0.95$, PPO clip range $\varepsilon = 0.3$, entropy coefficient $10^{-2}$, value-loss coefficient $0.5$, and reward scaling $1.0$. Evaluation rollouts are performed every $1.5\times 10^{7}$ environment steps. The complete list of hyperparameters is given in Table~\ref{tab:arena_forage_hparams}.

%%%%%%%%%%%%%%%%%%%%%%%%%%%%%%%%%%%%%%%%%%%%%%%%%%%%%%%%%%%%%%%%%%%%%%%%%%%%%%%%%%%%%%%
\subsection{Locomotion}
\label{app:locomotion}

\textbf{Task.} The locomotion task assesses how effectively the motion prior $p_\psi$ supports goal-directed control via residual fine-tuning, without recourse to the behavior prior. The motion prior $p_\psi$ and decoder $p_\theta$ are loaded frozen from the Stage~2 checkpoint (Table~\ref{tab:distill_hparams}); only the task policy and critic are trained. We do not use the behavior prior in this task: we use $b_t = 1$ at every step, so that the latent is always $\mathbf{z}_t = \boldsymbol{\mu}_t^m + \boldsymbol{\delta}_t$, isolating the quality of residual control of the motion prior. At the start of each episode the rat is placed at the world origin in a neutral standing pose facing the world $+x$ direction (identity body quaternion, $z = 0$), on a flat arena. The task is to maintain a fixed target forward speed $v^\star$ along the world $+x$ axis. In addition to the proprioception vector $\mathbf{x}_t$ from Section~\ref{sec:preliminaries}, the agent receives a task observation $\mathbf{o}_t$ comprising the previous action, and the rat's kinematic and touch sensors. The target speed $v^\star$ itself is a fixed task parameter and is not part of the observation.

\textbf{Reward.} The reward at each step is a piecewise-linear forward-velocity tracking signal,
\begin{equation*}
  r_t = \max\!\Big(0,\; 1 - |v^x_t - v^\star| / v^\star\Big),
\end{equation*}
where $v^x_t$ is the world-frame $+x$ component of the torso's linear velocity (subtree center-of-mass velocity of the torso body). The reward equals $1$ when $v^x_t = v^\star$ and decays linearly to $0$ at $v^x_t = 0$ and at $v^x_t = 2v^\star$.

\textbf{Termination.} An episode terminates early if (i) the rat has \emph{fallen}---defined as torso $z$ below $0.03$\,m or torso tilt from vertical exceeding $60^\circ$---or (ii) any state value becomes \texttt{NaN}. Episodes are otherwise truncated at a fixed length of $1{,}000$ control steps ($10$\,s at the $100$\,Hz control rate inherited from Stage~1).

\textbf{Network architecture.} The task policy is an MLP with hidden sizes $[2048, 1024, 512]$ that maps the concatenation of $\mathbf{x}_t$ and $\mathbf{o}_t$ (normalized by running mean and variance) to the mean and log-standard-deviation of a tanh-squashed diagonal Gaussian over the residual, so that $\boldsymbol{\delta}_t \in [-1, 1]^{16}$. The critic is a separate $[2048, 1024, 1024, 512]$ MLP estimating the state-value function. Both networks use SiLU activations with LayerNorm between hidden layers and are randomly initialized using LeCun initialization with zero bias. The motion prior $p_\psi$ and decoder $p_\theta$ remain frozen throughout this task; the encoder $q_\phi$ is not used.

\textbf{Training.} We train the task policy with PPO \cite{schulman_proximal_2017} for $4 \times 10^{8}$ environment steps across $1{,}024$ parallel environments. Each PPO update uses a $20$-step unroll, batch size $1{,}024$, and $4$ epochs of $16$ minibatches per batch, with Adam at learning rate $10^{-4}$ and gradient clipping at maximum norm $50.0$. We use discount $\gamma = 0.99$, GAE $\lambda = 0.95$, PPO clip range $\varepsilon = 0.3$, entropy coefficient $10^{-2}$, value-loss coefficient $0.5$, and reward scaling $1.0$. Evaluation rollouts are performed every $1.5 \times 10^{7}$ environment steps. The complete list of hyperparameters is given in Table~\ref{tab:maintain_velocity_hparams}.

%%%%%%%%%%%%%%%%%%%%%%%%%%%%%%%%%%%%%%%%%%%%%%%%%%%%%%%%%%%%%%%%%%%%%%%%%%%%%%%%%%%%%%%
\subsection{Rearing}
\label{app:rear}

\textbf{Task.} The walk-then-rear task tests whether residual control of the motion prior can drive a navigate-and-rear sequence terminated by a sparse success bonus. As in the locomotion task, the motion prior $p_\psi$ and decoder $p_\theta$ are loaded frozen from the Stage~2 checkpoint (Table~\ref{tab:distill_hparams}); only the task policy and critic are trained, and we use $b_t=1$ at every step so that the latent is always $\mathbf{z}_t = \boldsymbol{\mu}_t^m + \boldsymbol{\delta}_t$. At the start of each episode the rat is placed at the world origin in a neutral standing pose facing the world $+x$ direction (identity body quaternion, $z=0$), on a flat arena; a translucent target box is rendered at the goal location for visualization. The agent must walk to the horizontal target $(x^\star, y^\star)=(1.0,\, 0.0)$\,m and rear up so that its skull reaches a target height $z^\star = 0.2$\,m. A target pose is achieved when the skull is simultaneously within $0.01$\,m of $(x^\star, y^\star)$ horizontally and within $0.02$\,m of $z^\star$ vertically; success requires holding the target pose continuously for $1.0$\,s ($100$ control steps at $100$\,Hz). In addition to the proprioception vector $\mathbf{x}_t$ from Section~\ref{sec:preliminaries}, the agent receives a task observation $\mathbf{o}_t$ comprising the previous action, the rat's kinematic and touch sensors, the egocentric position of the world origin in the torso frame, and the egocentric position of the target in the torso frame.

\textbf{Reward.} The reward at each step combines a dense piecewise-linear shaping term and a sparse terminal bonus. Letting $d_t = \|\mathbf{p}^{\text{skull}}_t - \mathbf{p}^\star\|_2$ denote the Euclidean distance from the skull to the target, the dense term is
\begin{equation*}
  r^{\text{dist}}_t = \max\!\Big(0,\; 1 - d_t / m\Big),
\end{equation*}
with shaping margin $m = 1.0$\,m, so it equals $1$ when the skull is at the target and decays linearly to $0$ at $d_t = m$. The sparse bonus $r^{\text{succ}}_t = 500$ fires on the single step at which the running hold counter first reaches $100$ steps, and is $0$ otherwise.

\textbf{Termination.} An episode terminates if (i) the hold counter reaches $100$ steps (success), (ii) the rat has \emph{fallen}---defined as torso $z$ below $0.02$\,m or torso tilt from vertical exceeding $80^\circ$---or (iii) any state value becomes \texttt{NaN}. Episodes are otherwise truncated at a fixed length of $500$ control steps ($5$\,s at the $100$\,Hz control rate inherited from Stage~1).

\textbf{Network architecture.} The task policy is an MLP with hidden sizes $[2048, 1024, 512]$ that maps the concatenation of $\mathbf{x}_t$ and $\mathbf{o}_t$ (normalized by running mean and variance) to the mean and log-standard-deviation of a tanh-squashed diagonal Gaussian over the residual, so that $\boldsymbol{\delta}_t \in [-1, 1]^{16}$. The critic is a separate $[2048, 1024, 1024, 512]$ MLP estimating the state-value function. Both networks use SiLU activations with LayerNorm between hidden layers and are randomly initialized using LeCun initialization with zero bias. The motion prior $p_\psi$ and decoder $p_\theta$ remain frozen throughout this task; the encoder $q_\phi$ is not used.

\textbf{Training.} We train the task policy with PPO \cite{schulman_proximal_2017} for $6\times 10^{8}$ environment steps across $1{,}024$ parallel environments. Each PPO update uses a $20$-step unroll, batch size $1{,}024$, and $4$ epochs of $16$ minibatches per batch, with Adam at learning rate $5\times 10^{-5}$ and gradient clipping at maximum norm $50.0$. We use discount $\gamma=0.99$, GAE $\lambda=0.95$, PPO clip range $\varepsilon=0.3$, entropy coefficient $10^{-2}$, value-loss coefficient $0.5$, and reward scaling $1.0$. Evaluation rollouts are performed every $1.5\times 10^{7}$ environment steps. The complete list of hyperparameters is given in Table~\ref{tab:walk_rear_hparams}.

%%%%%%%%%%%%%%%%%%%%%%%%%%%%%%%%%%%%%%%%%%%%%%%%%%%%%%%%%%%%%%%%%%%%%%%%%%%%%%%%%%%%%%%
\FloatBarrier
\subsection{Hyperparameters}
\label{app:hyperparameters}
Our hyperparameter choices draw on prior work where applicable, with task-specific tuning otherwise. PPO hyperparameters follow Brax's defaults \cite{Freeman2021-fb} with minor adjustments to maximize task reward. Physics simulation and virtual rodent parameters are taken from MIMIC-MJX \cite{zhang_mimic-mjx_2025}, and online distillation hyperparameters follow PULSE \cite{luo_universal_2024} with small modifications identified via constrained grid search to better suit the virtual rodent setting. All remaining hyperparameters were selected through grid search.

%%%%%%%%%%%%%%%%%%%%%%%%%%%%%%%%%%%%%%%%%%%%
% Supp Table: Stage 1 hyperparams
%%%%%%%%%%%%%%%%%%%%%%%%%%%%%%%%%%%%%%%%%%%%
\begin{table}[!htbp]
  \centering
  \caption{Hyperparameters used for Stage 1 (train expert imitator) of the SCAMPER training pipeline.}
  \label{tab:imitation_mlp_hparams}
  \scriptsize
  \setlength{\tabcolsep}{3pt}
  \resizebox{0.8\textwidth}{!}{%
  \begin{tabular}{@{} p{1.9cm} l l @{}}
    \toprule
    \textbf{Group} & \textbf{Hyperparameter} & \textbf{Value} \\
    \midrule
    \multirow{13}{=}{Environment}
      & Solver                                   & Newton \\
      & Solver iterations                        & 5 \\
      & Line-search iterations                   & 5 \\
      & Max contacts ($n_\text{aconmax}$)        & 256 \\
      & Max constraint rows ($n_\text{jmax}$)    & 128 \\
      & Simulation timestep                      & $2 \times 10^{-3}$ s (500 Hz) \\
      & Control timestep                         & $1 \times 10^{-2}$ s (100 Hz) \\
      & Mocap rate                               & 50 Hz \\
      & Clip length                              & 250 frames \\
      & Reference horizon                        & 5 frames \\
      & Episode start-frame range                & $[0, 50]$ \\
      & Initial joint velocity                   & zeros \\
      & Train / eval clip split                  & 0.8 / 0.2 \\
    \cmidrule(l){1-3}
    \multirow{7}{=}{Domain randomization}
      & Floor friction                           & $\mathcal{U}(0.4, 1.0)$ \\
      & Static friction scale                    & $\mathcal{U}(0.9, 1.1)$ \\
      & Armature scale                           & $\mathcal{U}(1.0, 1.05)$ \\
      & Center-of-mass jitter (m)                & $\mathcal{U}(-0.05, 0.05)$ \\
      & Link mass scale                          & $\mathcal{U}(0.9, 1.1)$ \\
      & Torso mass jitter                        & $\mathcal{U}(-0.02, 0.02)$ \\
      & Initial $q_\text{pos}$ jitter            & $\mathcal{U}(-0.05, 0.05)$ \\
    \cmidrule(l){1-3}
    \multirow{10}{=}{Reward weights ($w$, scale)}
      & Root position                            & $w{=}1.0$, $\sigma{=}0.035$ \\
      & Root orientation (quaternion)            & $w{=}1.0$, $\sigma{=}20.0$ \\
      & Joint angles                             & $w{=}1.0$, $\sigma{=}1.41$ \\
      & Joint velocities                         & $w{=}0.0$, $\sigma{=}1.0$ \\
      & Body positions                           & $w{=}0.0$, $\sigma{=}0.25$ \\
      & End-effectors                            & $w{=}1.0$, $\sigma{=}0.03$ \\
      & Healthy torso $z$-range (m)              & $[0.0325, 0.5]$, $w{=}1.0$ \\
      & Control cost                             & $w{=}0.02$ \\
      & Control-diff cost                        & $w{=}0.02$ \\
      & Energy cost (clipped at 50)              & $w{=}0.01$ \\
    \cmidrule(l){1-3}
    \multirow{3}{=}{Termination}
      & Root-position deviation                  & $> 0.1$ m \\
      & Root-orientation deviation               & $> 60^\circ$ \\
      & Pose error (L2)                          & $> 4.5$ \\
    \cmidrule(l){1-3}
    \multirow{4}{=}{Network architecture}
      & Policy architecture                      & MLP (no latent bottleneck) \\
      & Policy hidden sizes                      & $[2048, 1024, 1024, 512]$ \\
      & Critic hidden sizes                      & $[2048, 1024, 1024, 1024, 1024, 512]$ \\
      & Observation normalization                & running mean/std \\
    \cmidrule(l){1-3}
    \multirow{15}{=}{PPO training}
      & Parallel environments                    & 4{,}096 \\
      & Total environment steps                  & $3 \times 10^{9}$ \\
      & Batch size                               & 1{,}024 \\
      & Minibatches per epoch                    & 16 \\
      & Updates per batch                        & 4 \\
      & Learning rate (Adam)                     & $1 \times 10^{-4}$ \\
      & Gradient clip threshold                  & 20.0 \\
      & PPO clip range $\varepsilon$             & 0.2 \\
      & Discount factor $\gamma$                 & 0.95 \\
      & GAE / unroll length                      & 20 \\
      & Action repeat                            & 1 \\
      & Entropy coefficient                      & $1 \times 10^{-2}$ \\
      & Value-loss coefficient                   & 0.25 \\
      & Reward scaling                           & 1 \\
      & Evaluation / reset interval              & $2.5 \times 10^{7}$ steps \\
    \bottomrule
  \end{tabular}
  }
\end{table}
%%%%%%%%%%%%%%%%%%%%%%%%%%%%%%%%%%%%%%%%%%%%

%%%%%%%%%%%%%%%%%%%%%%%%%%%%%%%%%%%%%%%%%%%%
% Supp Table: Stage 2 hyperparams
%%%%%%%%%%%%%%%%%%%%%%%%%%%%%%%%%%%%%%%%%%%%
\begin{table}[!htbp]
  \centering
  \caption{Hyperparameters used for Stage 2 (online distillation) of the SCAMPER
    training pipeline. Environment, termination, and domain-randomization
    settings are inherited from Stage 1 (Table~\ref{tab:imitation_mlp_hparams}).}
  \label{tab:distill_hparams}
  \scriptsize
  \setlength{\tabcolsep}{3pt}
  \resizebox{0.75\textwidth}{!}{%
  \begin{tabular}{@{} p{1.9cm} l l @{}}
    \toprule
    \textbf{Group} & \textbf{Hyperparameter} & \textbf{Value} \\
    \midrule
    \multirow{3}{=}{Distillation loss}
      & AR(1) smoothness weight $\alpha$            & $5 \times 10^{-3}$ \\
      & Gradient clip threshold (global norm)       & $10.0$ \\
      & Encoder / prior logvar clamp                & $[-5, 0]$ \\
    \cmidrule(l){1-3}
    \multirow{4}{=}{KL schedule}
      & Schedule type                               & linear ramp, then hold \\
      & Start weight $\lambda_0$                    & $1 \times 10^{-3}$ \\
      & End weight $\lambda_1$                      & $1 \times 10^{-4}$ \\
      & Ramp range (fraction of training)           & $[0\%,\, 25\%]$ \\
    \cmidrule(l){1-3}
    \multirow{2}{=}{Hard-negative mining}
      & Reweight strength $\alpha_{\text{HN}}$      & $0.1$ \\
      & Max sample weight $w_{\text{max}}$          & $5.0$ \\
    \cmidrule(l){1-3}
    \multirow{4}{=}{Network architecture}
      & Latent size                                 & $16$ \\
      & Encoder hidden sizes                        & $[2048, 1024, 1024]$ \\
      & Decoder hidden sizes                        & $[1024, 1024, 512]$ \\
      & Motion prior hidden sizes                   & $[1536, 1024, 512]$ \\
    \cmidrule(l){1-3}
    \multirow{10}{=}{Training}
      & Parallel environments                       & $4{,}096$ \\
      & Total environment steps                     & $1 \times 10^{9}$ \\
      & Batch size                                  & $1{,}024$ \\
      & Minibatches per epoch                       & $16$ \\
      & Updates per batch                           & $4$ \\
      & Learning rate (Adam)                        & $1 \times 10^{-4}$ \\
      & Unroll length                               & $20$ \\
      & Action repeat                               & $1$ \\
      & Train / eval clip split                     & $0.8 / 0.2$ \\
      & Evaluation / reset interval                 & $2 \times 10^{7}$ steps \\
    \bottomrule
  \end{tabular}
  }
\end{table}
%%%%%%%%%%%%%%%%%%%%%%%%%%%%%%%%%%%%%%%%%%%%

%%%%%%%%%%%%%%%%%%%%%%%%%%%%%%%%%%%%%%%%%%%%
% Supp Table: Stage 3 hyperparams
%%%%%%%%%%%%%%%%%%%%%%%%%%%%%%%%%%%%%%%%%%%%
\begin{table}[!htbp]
  \centering
  \caption{Hyperparameters used for Stage 3 (behavior prior distillation) of the
    SCAMPER training pipeline. Environment
    settings are inherited from Stage 1 (Table~\ref{tab:imitation_mlp_hparams});
    the encoder, decoder, and motion prior are loaded frozen from Stage 2
    (Table~\ref{tab:distill_hparams}). Only the behavior prior is updated.}
  \label{tab:behavior_prior_hparams}
  \scriptsize
  \setlength{\tabcolsep}{3pt}
  \resizebox{0.75\textwidth}{!}{%
  \begin{tabular}{@{} p{1.9cm} l l @{}}
    \toprule
    \textbf{Group} & \textbf{Hyperparameter} & \textbf{Value} \\
    \midrule
    \multirow{2}{=}{Distillation loss}
      & KL weight                                   & $1.0$ \\
      & Gradient clip threshold (global norm)       & $10.0$ \\
    \cmidrule(l){1-3}
    \multirow{2}{=}{Network architecture}
      & Latent size (inherited from Stage 2)        & $16$ \\
      & Behavior prior hidden sizes                 & $[2048, 2048, 1024, 512]$ \\
    \cmidrule(l){1-3}
    \multirow{10}{=}{Training}
      & Parallel environments                       & $4{,}096$ \\
      & Total environment steps                     & $2 \times 10^{9}$ \\
      & Batch size                                  & $1{,}024$ \\
      & Minibatches per epoch                       & $16$ \\
      & Updates per batch                           & $4$ \\
      & Learning rate (Adam)                        & $1 \times 10^{-4}$ \\
      & Unroll length                               & $20$ \\
      & Action repeat                               & $1$ \\
      & Train / eval clip split                     & $0.8 / 0.2$ \\
      & Evaluation / reset interval                 & $2 \times 10^{7}$ steps \\
    \bottomrule
  \end{tabular}
  }
\end{table}
%%%%%%%%%%%%%%%%%%%%%%%%%%%%%%%%%%%%%%%%%%%%

%%%%%%%%%%%%%%%%%%%%%%%%%%%%%%%%%%%%%%%%%%%%
% Supp Table: Arena foraging hyperparams
%%%%%%%%%%%%%%%%%%%%%%%%%%%%%%%%%%%%%%%%%%%%
\begin{table}[!htbp]
  \centering
  \caption{Hyperparameters used to train SCAMPER on the arena foraging task
    (Stage~4: transfer learning). The decoder, motion prior, and
    behavior prior are loaded frozen from the Stage~3 checkpoint
    (Table~\ref{tab:behavior_prior_hparams}); only the task policy and
    critic are trained at
    this stage.}
  \label{tab:arena_forage_hparams}
  \scriptsize
  \setlength{\tabcolsep}{3pt}
  \resizebox{0.75\textwidth}{!}{%
  \begin{tabular}{@{} p{1.9cm} l l @{}}
    \toprule
    \textbf{Group} & \textbf{Hyperparameter} & \textbf{Value} \\
    \midrule
    \multirow{9}{=}{Task / environment}
      & Episode length                              & $2{,}000$ control steps \\
      & Evaluation interval                         & $1.5 \times 10^{7}$ env steps \\
      & Number of treats per episode                & $10$ \\
      & Treat sphere radius                         & $0.06$ m \\
      & Treat height                                & $0.03$ m \\
      & Minimum spawn radius                        & Variable \\
      & Maximum spawn radius                        & $2.0$ m \\
      & Scent kernel $\sigma$                       & $0.2$ m \\
      & Treat-collection reward                     & $500.0$ \\
    \cmidrule(l){1-3}
    \multirow{3}{=}{Behavior--motion prior gating}
      & Gate threshold $b_\text{bias}$              & $0.5$ \\
      & Gate hysteresis margin                      & $0.2$ \\
      & Behavior-prior fixed log-variance           & $-2.25$ \\
    \cmidrule(l){1-3}
    \multirow{2}{=}{Network architecture}
      & Task policy hidden sizes                    & $[2048, 1024, 512]$ \\
      & Critic hidden sizes                         & $[2048, 1024, 1024, 512]$ \\
    \cmidrule(l){1-3}
    \multirow{15}{=}{PPO training}
      & Parallel environments                       & $1{,}024$ \\
      & Total environment steps                     & $2 \times 10^{9}$ \\
      & Batch size                                  & $1{,}024$ \\
      & Minibatches per epoch                       & $16$ \\
      & Updates per batch                           & $4$ \\
      & Learning rate (Adam)                        & $1 \times 10^{-4}$ \\
      & Gradient clip threshold (max norm)          & $50.0$ \\
      & PPO clip range $\varepsilon$                & $0.3$ \\
      & Discount factor $\gamma$                    & $0.99$ \\
      & GAE $\lambda$                               & $0.95$ \\
      & Entropy coefficient                         & $1 \times 10^{-2}$ \\
      & Value-loss coefficient                      & $0.5$ \\
      & Reward scaling                              & $1.0$ \\
      & Unroll length                               & $20$ \\
      & Action repeat                               & $1$ \\
    \bottomrule
  \end{tabular}
  }
\end{table}
%%%%%%%%%%%%%%%%%%%%%%%%%%%%%%%%%%%%%%%%%%%%

%%%%%%%%%%%%%%%%%%%%%%%%%%%%%%%%%%%%%%%%%%%%
% Supp Table: Locomotion hyperparams
%%%%%%%%%%%%%%%%%%%%%%%%%%%%%%%%%%%%%%%%%%%%
\begin{table}[!htbp]
  \centering
  \caption{Hyperparameters used to train SCAMPER on the locomotion task. The motion prior
    $p_\psi$ and decoder $p_\theta$ are loaded frozen from the Stage~2
    checkpoint (Table~\ref{tab:distill_hparams}); only the task policy and
    critic are trained at this stage. The policy outputs a residual
    $\boldsymbol{\delta}_t$ added to the (deterministic) motion-prior mean.}
  \label{tab:maintain_velocity_hparams}
  \scriptsize
  \setlength{\tabcolsep}{3pt}
  \resizebox{0.75\textwidth}{!}{%
  \begin{tabular}{@{} p{1.9cm} l l @{}}
    \toprule
    \textbf{Group} & \textbf{Hyperparameter} & \textbf{Value} \\
    \midrule
    \multirow{4}{=}{Task / environment}
      & Episode length                              & $1{,}000$ control steps \\
      & Evaluation interval                         & $1.5 \times 10^{7}$ env steps \\
      & Target forward speed                        & Variable \\
      & Forward-velocity reward weight              & $1.0$ \\
    \cmidrule(l){1-3}
    \multirow{2}{=}{Network architecture}
      & Task policy hidden sizes                    & $[2048, 1024, 512]$ \\
      & Critic hidden sizes                         & $[2048, 1024, 1024, 512]$ \\
    \cmidrule(l){1-3}
    \multirow{15}{=}{PPO training}
      & Parallel environments                       & $1{,}024$ \\
      & Total environment steps                     & $4 \times 10^{8}$ \\
      & Batch size                                  & $1{,}024$ \\
      & Minibatches per epoch                       & $16$ \\
      & Updates per batch                           & $4$ \\
      & Learning rate (Adam)                        & $1 \times 10^{-4}$ \\
      & Gradient clip threshold (max norm)          & $50.0$ \\
      & PPO clip range $\varepsilon$                & $0.3$ \\
      & Discount factor $\gamma$                    & $0.99$ \\
      & GAE $\lambda$                               & $0.95$ \\
      & Entropy coefficient                         & $1 \times 10^{-2}$ \\
      & Value-loss coefficient                      & $0.5$ \\
      & Reward scaling                              & $1.0$ \\
      & Unroll length                               & $20$ \\
      & Action repeat                               & $1$ \\
    \bottomrule
  \end{tabular}
  }
\end{table}
%%%%%%%%%%%%%%%%%%%%%%%%%%%%%%%%%%%%%%%%%%%%

%%%%%%%%%%%%%%%%%%%%%%%%%%%%%%%%%%%%%%%%%%%%
% Supp Table: Walk-rear hyperparams
%%%%%%%%%%%%%%%%%%%%%%%%%%%%%%%%%%%%%%%%%%%%
\begin{table}[!htbp]
  \centering
  \caption{Hyperparameters used to train SCAMPER on the walk-rear task. The
    motion prior $p_\psi$ and decoder $p_\theta$ are loaded frozen from the
    Stage~2 checkpoint (Table~\ref{tab:distill_hparams}); only the task policy
    and critic are trained. The policy outputs a residual $\boldsymbol{\delta}_t$
    added to the (deterministic) motion-prior mean.}
  \label{tab:walk_rear_hparams}
  \scriptsize
  \setlength{\tabcolsep}{3pt}
  \resizebox{0.75\textwidth}{!}{%
  \begin{tabular}{@{} p{1.9cm} l l @{}}
    \toprule
    \textbf{Group} & \textbf{Hyperparameter} & \textbf{Value} \\
    \midrule
    \multirow{9}{=}{Task / environment}
      & Episode length                              & $500$ control steps \\
      & Evaluation interval                         & $1.5 \times 10^{7}$ env steps \\
      & Target position $(x^\star, y^\star, z^\star)$ & $(1.0,\, 0.0,\, 0.2)$ m \\
      & $xy$-tolerance for success                  & $0.01$ m \\
      & $z$-tolerance for success                   & $0.02$ m \\
      & Required hold duration                      & $1.0$ s \\
      & Distance-shaping margin                     & $1.0$ m \\
      & Distance-to-target reward weight            & $1.0$ \\
      & Walk-rear success reward                    & $500.0$ \\
    \cmidrule(l){1-3}
    \multirow{2}{=}{Network architecture}
      & Task policy hidden sizes                    & $[2048, 1024, 512]$ \\
      & Critic hidden sizes                         & $[2048, 1024, 1024, 512]$ \\
    \cmidrule(l){1-3}
    \multirow{15}{=}{PPO training}
      & Parallel environments                       & $1{,}024$ \\
      & Total environment steps                     & $6 \times 10^{8}$ \\
      & Batch size                                  & $1{,}024$ \\
      & Minibatches per epoch                       & $16$ \\
      & Updates per batch                           & $4$ \\
      & Learning rate (Adam)                        & $5 \times 10^{-5}$ \\
      & Gradient clip threshold (max norm)          & $50.0$ \\
      & PPO clip range $\varepsilon$                & $0.3$ \\
      & Discount factor $\gamma$                    & $0.99$ \\
      & GAE $\lambda$                               & $0.95$ \\
      & Entropy coefficient                         & $1 \times 10^{-2}$ \\
      & Value-loss coefficient                      & $0.5$ \\
      & Reward scaling                              & $1.0$ \\
      & Unroll length                               & $20$ \\
      & Action repeat                               & $1$ \\
    \bottomrule
  \end{tabular}
  }
\end{table}
%%%%%%%%%%%%%%%%%%%%%%%%%%%%%%%%%%%%%%%%%%%%

%%%%%%%%%%%%%%%%%%%%%%%%%%%%%%%%%%%%%%%%%%%%%%%%%%%%%%%%%%%%%%%%%%%%%%%%%%%%%%%%%%%%%%%
%%%%%%%%%%%%%%%%%%%%%%%%%%%%%%%%%%%%%%%%%%%%%%%%%%%%%%%%%%%%%%%%%%%%%%%%%%%%%%%%%%%%%%%
\FloatBarrier
\section{Analysis of randomly generated behavior}

%%%%%%%%%%%%%%%%%%%%%%%%%%%%%%%%%%%%%%%%%%%%%%%%%%%%%%%%%%%%%%%%%%%%%%%%%%%%%%%%%%%%%%%
\subsection{Rollout collection protocol}
\label{app:rollout-collection}

All quantitative comparisons in this section (KID, Vendi score, root-height ECDF, spatial occupancy) are computed on rollouts collected with a single shared protocol so that the only difference between conditions is the controller that produced the trajectory.

For every condition, we start from the same set of initial states: the first frame of each of the $842$ reference clips, with the rat placed in the corresponding generalized pose and zero joint velocity. This matches the initialization used during imitation training and keeps the starting distribution of body configurations identical across conditions. Each rollout is then advanced in MuJoCo MJX under the chosen controller for the same horizon as the reference clip ($\sim$$5$\,s of simulated time, with control at $100$\,Hz on top of a $500$\,Hz physics step); the trajectory is driven purely by stochastic sampling of the controllers. At each step we log the rat's generalized coordinates $\mathbf{q}\in\mathbb{R}^{74}$ and the latent vector passed to the decoder. Clips are processed in parallel batches on a single GPU, and one randomization key is split per condition so that stochastic samples are reproducible and independent across conditions. We perform the rollouts using both the behavior prior and the motion prior and compare to the imitation rollouts as the ground truth. We test across both the deterministic sampling condition (no noise; $\sigma=0$) and an array of different sampling noise levels $\text{logvar}=\{-4, -3, -2, -1, -0.5, 0\}$ ($\sigma\approx\{0.14, 0.22, 0.37, 0.61, 0.78, 1.00\}$).

%%%%%%%%%%%%%%%%%%%%%%%%%%%%%%%%%%%%%%%%%%%%%%%%%%%%%%%%%%%%%%%%%%%%%%%%%%%%%%%%%%%%%%%
\FloatBarrier
\subsection{Vendi score}
\label{app:vendi}

\paragraph{What the Vendi score measures.} The Vendi score \citep{friedman_vendi_2023} quantifies the effective number of distinct modes in a sample under a user-specified similarity kernel. It is defined as the exponential of the von Neumann (Shannon) entropy of the eigenvalues of the normalized similarity matrix and, in contrast to a clustering-based count, requires no reference distribution, no labels, and no choice of cluster number. It reduces to the number of items when they are mutually orthogonal and to $1$ when they are all identical. Higher Vendi score indicates greater behavioral diversity.

\paragraph{Egocentric kinematic features.} Unlike for KID, where the feature extractor is a learned VAE, we measure diversity directly in an egocentric kinematic feature space so that the score is interpretable in terms of physically meaningful movement variation rather than the geometry of a learned latent space. From each rollout's generalized coordinates $\mathbf{q}\in\mathbb{R}^{74}$, we compute, per frame, the root height, root roll and pitch, body-frame linear velocity, yaw rate, and the joint angles --- $74$ dims/frame. The body-frame velocities are intrinsically heading-invariant and the planar root translation is dropped, so the representation is invariant to the absolute location of the rat in the arena. Frames are then stacked into sliding windows of length $30$ ($0.6$\,s at the $50$\,Hz control rate) with a stride of $15$, yielding $2{,}220$-dimensional feature vectors per window.

\paragraph{Joint PCA(50) embedding.} To make the Vendi computation tractable and to apply a consistent similarity geometry across conditions, we fit a single PCA basis of dimension $50$. The basis is fit jointly on a stratified subsample of $50{,}000$ windows, drawn proportionally across every condition (each prior at every noise level, plus the training-data baseline) so that no condition dominates the principal directions. All windows are then projected into this shared $50$-dimensional embedding before the Vendi score is computed.

\paragraph{Computing the Vendi score.} Let $X \in \mathbb{R}^{n \times 50}$ be the PCA-embedded windows of a single rollout sample, and let $X_{n}$ denote $X$ with rows $\ell_2$-normalized. We use the cosine similarity kernel and exploit the dual identity $\mathrm{eig}(X_{n} X_{n}^{\top}/n) = \mathrm{eig}(X_{n}^{\top} X_{n}/n)$ (up to zero eigenvalues) to compute Vendi via the small $50 \times 50$ Gram matrix,
\begin{equation*}
  M = \frac{1}{n}\, X_{n}^{\top} X_{n}, \qquad \{\lambda_{i}\}_{i=1}^{50} = \mathrm{eigvalsh}(M),
\end{equation*}
which is exact for the cosine Vendi score and avoids forming the $n \times n$ kernel. The Vendi score is then the exponential of the Shannon entropy of the (positive) eigenvalues,
\begin{equation*}
  \mathrm{VS}(X) = \exp\!\left( -\sum_{i:\, \lambda_{i} > 0} \lambda_{i}\, \log \lambda_{i} \right),
\end{equation*}
with the convention $0 \log 0 = 0$.

\paragraph{Bootstrap over clips.} The independent unit of replication is the imitation \emph{clip} (windows within a clip share a starting condition and are not independent), so we bootstrap at the clip level. There are $842$ clips per condition (motion-prior, behavior-prior, training-data). For each condition we bootstrap by resampling the clip indices with replacement and recompute the Vendi score. We perform $B = 500$ bootstrap iterations. The same clip indices are reused across all conditions in a given iteration, so cross-policy comparisons are exactly paired. We report bootstrap means with 95\% confidence intervals.

%%%%%%%%%%%%%%%%%%%%%%%%%%%%%%%%%%%%%%%%%%%%%%%%%%%%%%%%%%%%%%%%%%%%%%%%%%%%%%%%%%%%%%%
\subsection{Kernel Inception Distance}
\label{app:kid}

\paragraph{What KID measures.} The Kernel Inception Distance (KID) quantifies the dissimilarity between two distributions in a learned feature space as the squared Maximum Mean Discrepancy ($\mathrm{MMD}^2$) under a polynomial kernel \citep{binkowski_demystifying_2021}. In contrast to a metric like Fr\'echet Inception Distance (FID) \cite{heusel_gans_2018}, it is an unbiased estimator and remains well-defined for moderately sized samples. KID is zero when the two distributions match.

\paragraph{Training the feature extractor.} There is no pre-trained ``Inception'' network for rodent kinematics, so we train a small variational autoencoder (VAE) on the real (training-data) joint-angle trajectories and use its encoder mean $\mu_{\phi}(\mathbf{x}) \in \mathbb{R}^{d}$ as the feature extractor. The VAE is optimized by the standard $\beta$-VAE objective
\begin{equation*}
  \mathcal{L}_{\mathrm{VAE}}(\phi, \theta) = \mathbb{E}_{\mathbf{x}}\!\left[ \|\hat{\mathbf{x}}_{\theta}(\mathbf{x}) - \mathbf{x}\|_{2}^{2} + \beta\, D_{\mathrm{KL}}\!\left( q_{\phi}(\mathbf{z}\mid\mathbf{x}) \,\|\, \mathcal{N}(\mathbf{0},\mathbf{I}) \right) \right],
\end{equation*}
with hyperparameters listed in Table~\ref{tab:kid-vae-hparams}. After training, the decoder is discarded and only $\mu_{\phi}$ is retained; we repeat training under several random seeds and report the mean and standard deviation of KID across them.

\paragraph{Trajectory inputs and preprocessing.} Each training sample is a single 5\,s trajectory of the rat's generalized coordinates $\mathbf{q}\in\mathbb{R}^{74}$ --- planar root position $(x,y)$, root height $z$, root quaternion ($4$ dims), and $67$ joint angles --- drawn from the training-data imitation rollouts. We use imitation rollouts as the ground truth as opposed to the original mocap because the alignment process from keypoints to joint angle trajectories (through STAC-MJX \cite{zhang_mimic-mjx_2025}) produces jagged transition frame-to-frame. Imitation of this reference allows the simulated physics to smooth out the trajectories and results in a fairer comparison to the generated motion. Before encoding, three operations are applied so that the comparison between training and generated distributions is invariant to nuisance factors and on a consistent scale: (i) the root quaternion is replaced by a yaw-removed $6$D rotation representation \citep{zhou_continuity_2020}; (ii) the planar root position $(x,y)$ is discarded; (iii) joint dimensions are $z$-score normalized using statistics computed on the real data. Each per-step vector therefore consists of root height, root rotation, and $67$ normalized joint angles for $74$ features per step; the full trajectory is flattened to a single vector before being fed to the encoder. The same preprocessing is applied to every generated rollout.

\paragraph{Computing KID.} Let $X = \{\mu_{i}^{r}\}_{i=1}^{m}$ and $Y = \{\mu_{j}^{f}\}_{j=1}^{m}$ be the encoder features of the real and generated samples. We use the cubic polynomial kernel of \citet{binkowski_demystifying_2021},
\begin{equation*}
  k(\mathbf{u},\mathbf{v}) = \left( \frac{1}{d}\, \mathbf{u}^{\top}\mathbf{v} + 1 \right)^{\!3},
\end{equation*}
and define $\widehat{\mathrm{MMD}}_u^{2}$ as the unbiased squared-MMD estimator of \citet{gretton_kernel_2012},
\begin{equation*}
  \widehat{\mathrm{MMD}}_u^{2}(X, Y) = \frac{1}{m(m-1)} \sum_{i \neq j} k(\mu_{i}^{r}, \mu_{j}^{r}) + \frac{1}{m(m-1)} \sum_{i \neq j} k(\mu_{i}^{f}, \mu_{j}^{f}) - \frac{2}{m^{2}} \sum_{i,j} k(\mu_{i}^{r}, \mu_{j}^{f}).
\end{equation*}
To stabilize the per-seed estimate, we average $\widehat{\mathrm{MMD}}_u^{2}$ over $S$ random subsets of size $m$ drawn independently from $X$ and $Y$:
\begin{equation*}
  \mathrm{KID}(X,Y) = \frac{1}{S} \sum_{s=1}^{S} \widehat{\mathrm{MMD}}_u^{2}\!\left( X^{(s)}, Y^{(s)} \right).
\end{equation*}

%%%%%%%%%%%%%%%%%%%%%%%%%%%%%%%%%%%%%%%%%%%%
% Supp Table: VAE feature extractor hyperparams
%%%%%%%%%%%%%%%%%%%%%%%%%%%%%%%%%%%%%%%%%%%%
\begin{table}[!htbp]
  \centering
  \caption{Hyperparameters for the VAE feature extractor (``inception network'') used to compute KID.}
  \label{tab:kid-vae-hparams}
  \scriptsize
  \setlength{\tabcolsep}{3pt}
  \resizebox{0.7\textwidth}{!}{%
  \begin{tabular}{@{} p{2.4cm} l l @{}}
    \toprule
    \textbf{Group} & \textbf{Hyperparameter} & \textbf{Value} \\
    \midrule
    \multirow{5}{=}{Architecture}
      & Encoder hidden sizes                     & $[512, 256]$ \\
      & Decoder hidden sizes                     & $[256, 512]$ \\
      & Latent dimensionality $d$                & $64$ \\
      & Activation                               & SiLU \\
      & Layer normalization                      & yes \\
    \cmidrule(l){1-3}
    \multirow{2}{=}{Regularization}
      & Dropout rate                             & $0.1$ \\
      & Weight decay (AdamW)                     & $1 \times 10^{-5}$ \\
    \cmidrule(l){1-3}
    \multirow{2}{=}{$\beta$-VAE}
      & Target $\beta$                           & $1.0$ \\
      & $\beta$ warmup (linear)                  & $2{,}500$ epochs \\
    \cmidrule(l){1-3}
    \multirow{6}{=}{Training}
      & Optimizer                                & AdamW \\
      & Learning rate                            & $1 \times 10^{-3}$ \\
      & Gradient clip norm                       & $1.0$ \\
      & Batch size                               & $128$ \\
      & Total epochs                             & $5{,}000$ \\
      & Random seeds                             & $\{0, 1, 2, 3\}$ \\
    \cmidrule(l){1-3}
    \multirow{3}{=}{KID computation}
      & Polynomial kernel degree                 & $3$ \\
      & Number of subsets $S$                    & $100$ \\
      & Subset size $m$                          & $\min(|X|, |Y|, 1{,}000)$ \\
    \bottomrule
  \end{tabular}
  }
\end{table}
%%%%%%%%%%%%%%%%%%%%%%%%%%%%%%%%%%%%%%%%%%%%

%%%%%%%%%%%%%%%%%%%%%%%%%%%%%%%%%%%%%%%%%%%%%%%%%%%%%%%%%%%%%%%%%%%%%%%%%%%%%%%%%%%%%%%
\section{Gait analysis}
\label{app:gait-analysis}

We compare the gaits produced by SCAMPER, MIMIC-MJX, and the from-scratch baseline against empirical rat gait data from \citet{danner_spinal_2023}, at three target speeds spanning walk, trot, and gallop. The same pipeline is applied to every model rollout so that the only difference between conditions is the controller.

\paragraph{Rollouts.} For each (model, target speed) condition we run a fixed-length deterministic rollout of the locomotion task, with the velocity setpoint chosen to elicit the corresponding gait. The simulator state, the foot-body world positions, and the four limb touch-sensor traces (palm and sole on each side) are saved for offline analysis.

\paragraph{Contact barcode.} For each rollout we convert the four raw touch-sensor traces into a binary stance/swing barcode of shape $[4, T]$, with rows ordered $[\text{RH}, \text{RF}, \text{LF}, \text{LH}]$. The pipeline is: (i) per-channel median filtering to suppress sensor noise; (ii) an optional spike-removal step that zeros out brief above-threshold transients on a sensor channel that is otherwise quiescent (so that fleeting contacts never get registered as stance bouts); (iii) per-channel thresholding to a binary stance signal, with the threshold tuned per limb to account for differing typical contact forces; and (iv) binary morphological closing on each channel to fill brief gaps in otherwise-continuous stance bouts. The four cleaned binary traces are then stacked into the barcode used for all subsequent steps.

\paragraph{Gait cycle detection.} Gait cycles are extracted from the barcode using a chosen reference limb. We find every stance onset (a $0\!\to\!1$ transition on the reference channel), discard onsets that are closer together than a minimum cycle length (rejects double-detection of a single foot-strike), and define each gait cycle as the interval between two consecutive onsets that is at least the minimum cycle length. Detection is restricted to a per-condition analysis window so that the rat is at steady-state speed during the cycles being analyzed.

\paragraph{Per-cycle metrics.} For every detected cycle we compute the standard set of behavioral-gait metrics used in rodent locomotion studies \cite{hamers_automated_2001, mendes_quantification_2015}:
\begin{itemize}
  \item \textbf{Stride length.} Euclidean distance between the reference foot's world position at the start and end of the cycle.
  \item \textbf{Mean velocity.} Stride length divided by cycle duration.
  \item \textbf{Cadence.} Stride frequency converted to steps per minute (counting both feet of a contralateral pair).
  \item \textbf{Stance percentage / duty factor.} Fraction of the cycle during which the reference limb is in stance.
\end{itemize}
We then aggregate these metrics across all cycles in the analysis window and report mean and standard deviation per (model, gait) condition.

\paragraph{Swing--stance gait diagrams.} The diagrams in Figure \ref{fig:maintain-vel}c are computed by isolating, on each of the four limb channels, every complete stance bout within the analysis window (partial bouts that overlap the window edges are dropped), then computing per-limb mean stance duration and standard deviation, mean stride period (median inter-onset interval averaged over limbs), and the phase offset of each limb relative to a reference hind limb. Phase offsets are taken as the mean delay between the first stance onset of each limb after a reference-limb onset, modulo the stride period. The resulting per-limb stance bars are tiled across the cycle for visualization.

\paragraph{Empirical reference.} The empirical rat data \citep{danner_spinal_2023} are filtered to intact-spinal-cord animals for each of the three gait categories, and we extract per-stride duty factors for the four limbs, the inter-limb phase relationships (left--right hind, homolateral, and diagonal), and the stride duration. These are converted into the same set of per-limb mean stance, phase offset, mean stride period, and duty-factor summaries used for the model rollouts, providing a directly comparable reference distribution.

\paragraph{Quantitative summary.} For the comparison table, every metric is summarized per (gait, model) as $\mu \pm \sigma$ over cycles. Each model is then scored against the empirical reference as the difference in means normalized by the empirical standard deviation, $|\Delta\mu| / \sigma_{\text{emp}}$, which expresses the discrepancy in units of the natural across-stride variability of real rats. The metrics included are stride duration, step frequency, per-limb stance and swing duration, per-limb duty factor, temporal symmetry (left/right hind duty-factor asymmetry), and per-limb phase offset --- giving a single table that aggregates the full per-limb timing pattern alongside the cycle-level summaries.

%%%%%%%%%%%%%%%%%%%%%%%%%%%%%%%%%%%%%%%%%%%%%%%%%%%%%%%%%%%%%%%%%%%%%%%%%%%%%%%%%%%%%%%
%%%%%%%%%%%%%%%%%%%%%%%%%%%%%%%%%%%%%%%%%%%%%%%%%%%%%%%%%%%%%%%%%%%%%%%%%%%%%%%%%%%%%%%
\FloatBarrier
\section{Supplementary analyses and results}

%%%%%%%%%%%%%%%%%%%%%%%%%%%%%%%%%%%%%%%%%%%%%%%%%%%%%%%%%%%%%%%%%%%%%%%%%%%%%%%%%%%%%%%
\subsection{Empirical tests of Propositions 1 and 2}
\label{app:emp-just}

To evaluate the encoder and the two priors on the states they actually encounter, we rolled out the reference dataset in closed loop through the frozen Stage-2 encoder--decoder and collected \textasciitilde{}200{,}000 proprioceptive states. From these we drew 20{,}000 anchor states. For each anchor we identified the 32 most similar proprioceptive states in a 64-dimensional principal-component space, excluding the anchor's own clip to avoid autocorrelation, and took the goal that followed each of them as a candidate goal. Each anchor's proprioception was then held fixed while its goal was swapped for each of these $K=32$ candidates.

Both propositions rest on the identity $\mathbb{E}\left[D_{\mathrm{KL}}(q_\phi \,\|\, p)\right] = I_\phi(Z;G \mid X) + \mathbb{E}\left[D_{\mathrm{KL}}(m_\phi \,\|\, p)\right]$, where $m_\phi$ is the encoder marginal. The first term depends only on the encoder, so it is identical for both priors and cancels in their difference. The comparison between priors therefore reduces to closed-form Gaussian KL divergences. The behavior prior $p_\varphi$ is 8.46 nats closer to the encoder marginal than the motion prior $p_\psi$ averaged across all individual states (Fig.~\ref{fig:prop-supp}a).

To measure how much of the encoder's behavior at a fixed body state is driven by the goal, we decompose the encoder's latent variance at each anchor with the law of total variance. The between-goal term is the variance of the encoder mean across the 32 candidate goals, $\mathrm{Var}_g[\mu_\phi]$, and the within-goal term is the expected encoder variance, $\mathbb{E}_g[\sigma_\phi^2]$. With proprioception held fixed, 49.7\% of the encoder's latent variance is between-goal (Fig.~\ref{fig:prop-supp}b). The regularization in Proposition 1 therefore limits the conditional mutual information $I_\phi(Z;G \mid X)$ but does not drive it to zero. Driving it to zero would be undesirable, because the behavior prior distilled from the encoder could then no longer produce diverse motion.

Because all 32 candidates at an anchor share the same proprioceptive state, any above-chance decoding of the goal reflects information carried by the latent alone. We draw one candidate uniformly at random, sample a latent from the encoder's distribution for that candidate, score all 32 candidates by the density their encoder Gaussians assign to the sample, and predict the highest-scoring one. Since the candidate prior is uniform and the densities are the encoder's own, this rule is Bayes-optimal and needs no trained probe. It identifies the correct goal 69.4\% of the time, against 3.1\% chance (a 22-fold improvement), and ranks it in the top 5 for 94.6\% of anchors, against 15.6\% chance (Fig.~\ref{fig:prop-supp}c).

Finally, we tested Proposition 1 directly by repeating the fixed-proprioception analysis on all 50 Stage-2 checkpoints, each with its own observation normalizer. We measure goal sensitivity as the mutual information between the latent and which of the $K$ candidate goals was used, which is bounded above by $\log K$. The Stage-2 KL coefficient $\lambda$ ramps linearly down from $10^{-3}$ to $10^{-4}$ over the first 25\% of training and is then held at its floor. Goal sensitivity tracks this schedule. It falls from 1.70 to 1.33 nats while $\lambda$ is near its maximum, which is the bottleneck Proposition 1 describes. It then rises monotonically to 2.52 nats as $\lambda$ decays and stops changing when $\lambda$ reaches its floor, varying by less than 0.03 nats over the remaining 75\% of training (Fig.~\ref{fig:prop-supp}d). Nothing else in the training schedule changes at that point, so the strength of the bottleneck, rather than generic training dynamics, sets how much goal information the latent carries. Proposition 1 should therefore be read as describing a pressure that acts during training rather than a property of the final model: the schedule deliberately relaxes the bottleneck to improve reconstruction, and the converged encoder is goal-rich.

%%%%%%%%%%%%%%%%%%%%%%%%%%%%%%%%%%%%%%%%%%%%
% Supp 1: Empirical tests
% A: KL to encoder marginal
% B: Encoder latent variance
% C: Goal decoding accuracy
% D: Goal sensitivity
%%%%%%%%%%%%%%%%%%%%%%%%%%%%%%%%%%%%%%%%%%%%
\begin{figure}[!htbp]
  \centering
  \includegraphics[width=0.85\textwidth]{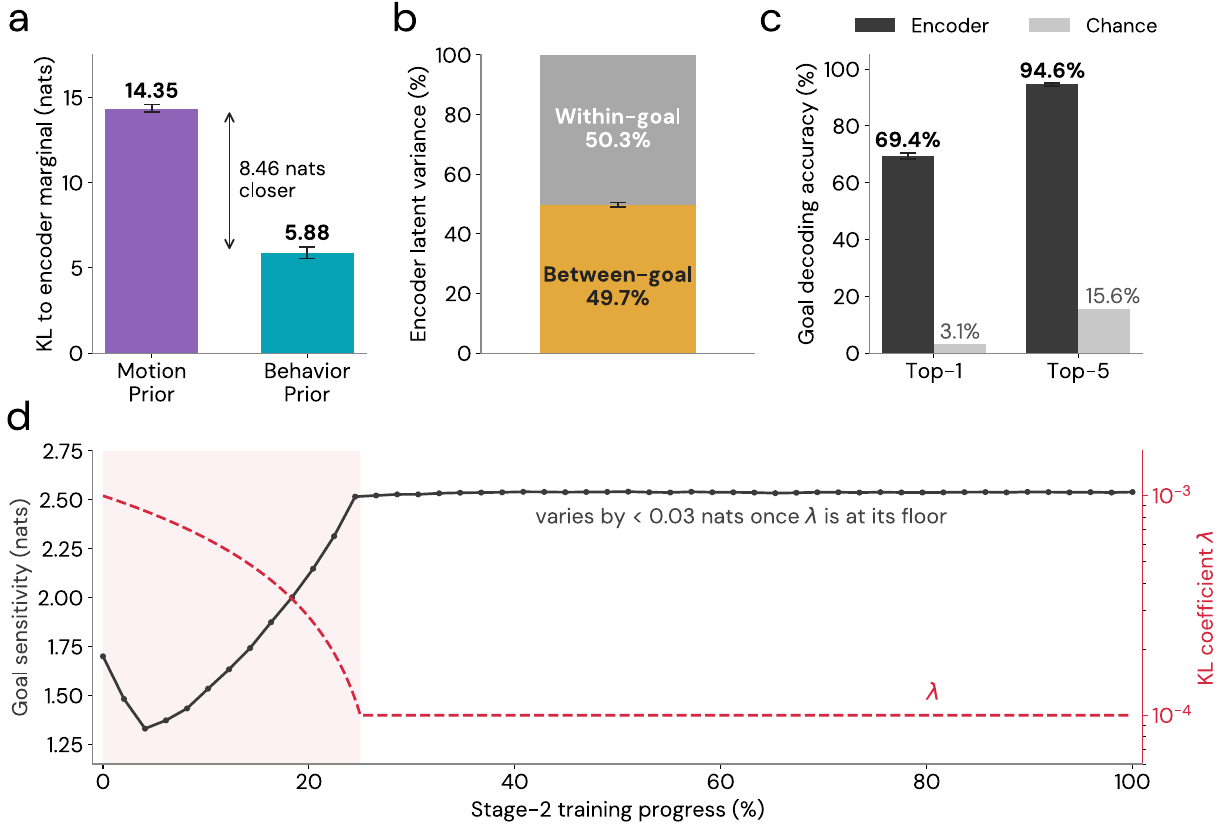}
  \caption{\textbf{(a)} KL divergence to the encoder marginal for the motion and behavior priors.
  \textbf{(b)} Between-goal and within-goal shares of the encoder's latent variance at fixed proprioception. \textbf{(c)} Goal decoding accuracy from the
  latent, compared with chance. \textbf{(d)} Goal sensitivity across Stage-2 training, with the KL coefficient $\lambda$ on the right axis. Error bars are
  95\% clip-level bootstrap CIs.}
  \label{fig:prop-supp}
\end{figure}
%%%%%%%%%%%%%%%%%%%%%%%%%%%%%%%%%%%%%%%%%%%%

%%%%%%%%%%%%%%%%%%%%%%%%%%%%%%%%%%%%%%%%%%%%%%%%%%%%%%%%%%%%%%%%%%%%%%%%%%%%%%%%%%%%%%%
\subsection{Diversity and realism of prior sampling}

In Fig.~\ref{fig:behavior-diversity} we demonstrate that when sampling random motion trajectories, we obtain greater behavioral diversity from the behavior prior than from the motion prior quantified by the resulting Vendi scores. However, the observed difference in diversity could reflect discrepancies in effective perturbation magnitude of the latent rather than qualitatively different distributions. To address this, we compare the behavioral diversity by examining the impact of noise as a function of both the scales and the jitter of the sampled latents. To disambiguate whether the scale of the latents is in fact responsible for the difference in diversity, we measure the Vendi score against the mean per-dimension standard deviation of the sampled latents for both priors across all rollouts at various sampling noise levels (Fig.~\ref{fig:supp-behavior-diversity-control}a). While we do find that the scale of the latents is different for the two priors, when accounting for this along the domain of the measured latent variance, the behavior prior is still more diverse than the motion prior. Secondly, to disambiguate whether the behavior prior's increased diversity is a result of jitter (i.e., how quickly the latents change from timestep to timestep), we measure the Vendi score versus the one-step latent displacement (Fig.~\ref{fig:supp-behavior-diversity-control}b). While we again find that both priors tend towards varying rates of per-step latent displacement for the same sampled noise level, when accounting for this difference, the behavior prior remains more behaviorally diverse than the motion prior over the domain of the latent jitter measured.

%%%%%%%%%%%%%%%%%%%%%%%%%%%%%%%%%%%%%%%%%%%%
% Supp 2: Latent controls
% A: Per-dimension latent s.d.
% B: Latent displacement
%%%%%%%%%%%%%%%%%%%%%%%%%%%%%%%%%%%%%%%%%%%%
\begin{figure}[!htbp]
  \centering
  \includegraphics[width=0.85\textwidth]{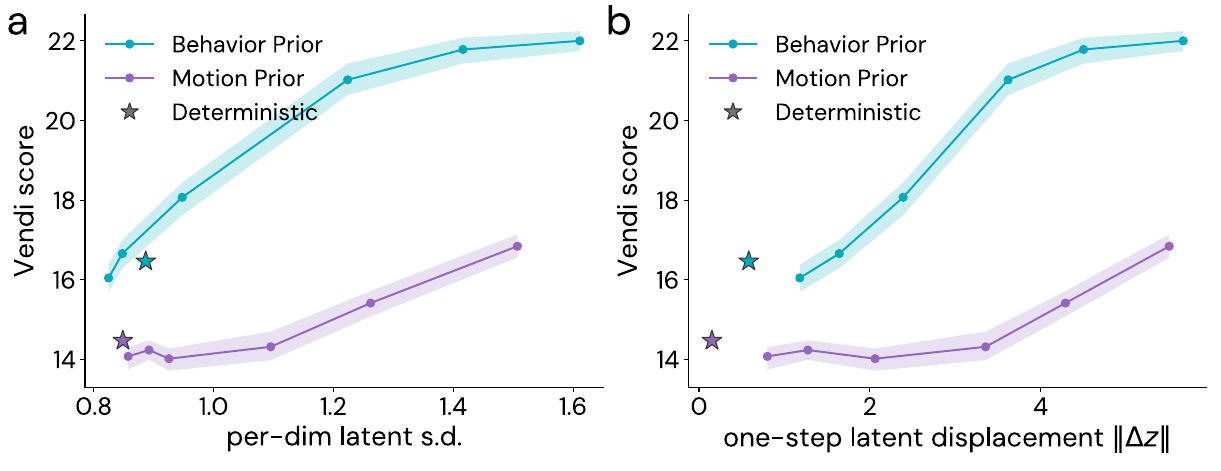}
  \caption{Diversity at matched latent scale. Vendi score of random generations vs.\ the realized per-dimension latent s.d.\ (\textbf{a}) and mean one-step latent displacement (\textbf{b}) at various levels of sampling noise ($\sigma$). Stars: deterministic ($\sigma=0$). 95\% bootstrap CIs shown as shaded regions.}
  \label{fig:supp-behavior-diversity-control}
\end{figure}
%%%%%%%%%%%%%%%%%%%%%%%%%%%%%%%%%%%%%%%%%%%%

%%%%%%%%%%%%%%%%%%%%%%%%%%%%%%%%%%%%%%%%%%%%
% Supp 3: Heatmaps
%%%%%%%%%%%%%%%%%%%%%%%%%%%%%%%%%%%%%%%%%%%%
\begin{figure}[!htbp]
  \centering
  \includegraphics[width=0.85\textwidth]{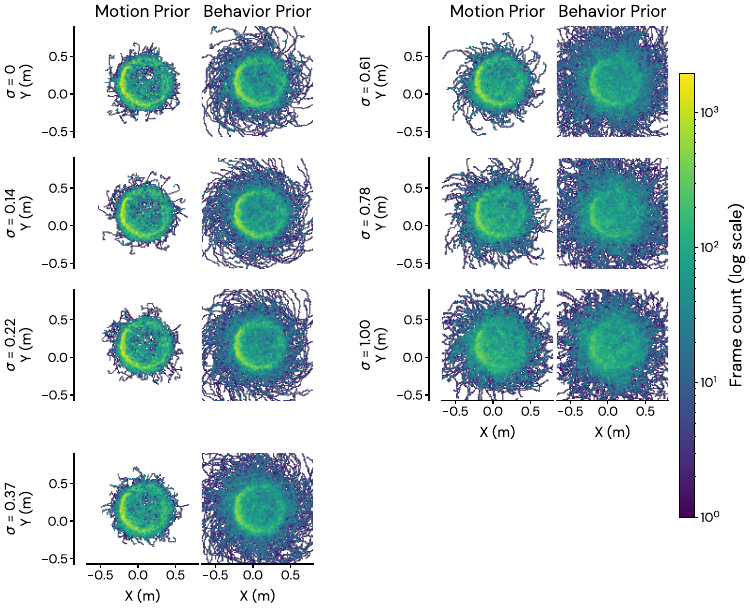}
  \caption{Spatial occupancy of all tested sampling noise levels plotted as 2D histograms of body location (log-scaled frame counts) for the motion prior and behavior prior.}
  \label{fig:supp-heatmaps}
\end{figure}
%%%%%%%%%%%%%%%%%%%%%%%%%%%%%%%%%%%%%%%%%%%%

%%%%%%%%%%%%%%%%%%%%%%%%%%%%%%%%%%%%%%%%%%%%
% Supp 4: ECDFs of rearing
%%%%%%%%%%%%%%%%%%%%%%%%%%%%%%%%%%%%%%%%%%%%
\begin{figure}[!htbp]
  \centering
  \includegraphics[width=0.5\textwidth]{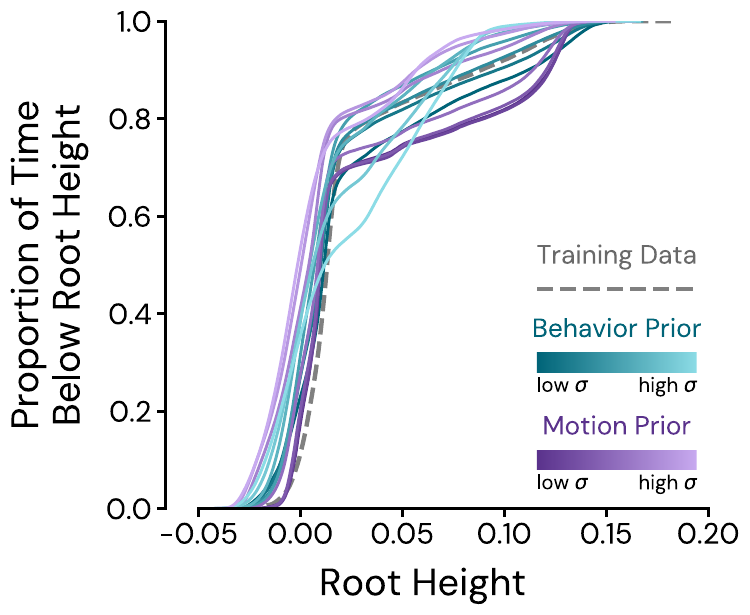}
  \caption{Empirical cumulative distribution function of root height at all tested sampling noise levels compared against the training data reference (\textit{dashed gray}).}
  \label{fig:supp-rearing}
\end{figure}
%%%%%%%%%%%%%%%%%%%%%%%%%%%%%%%%%%%%%%%%%%%%

%%%%%%%%%%%%%%%%%%%%%%%%%%%%%%%%%%%%%%%%%%%%
% Supp Table: Wasserstein distances for ECDF plots
%%%%%%%%%%%%%%%%%%%%%%%%%%%%%%%%%%%%%%%%%%%%
\begin{table}[!htbp]
  \centering
  \caption{Wasserstein distance between empirical root-height distributions of prior rollouts and the training data, across sampling noise levels $\sigma$. Corresponds to results from Figure~\ref{fig:supp-rearing}.}
  \label{tab:wasserstein-rear}
  \small
  \setlength{\tabcolsep}{5pt}
  \begin{tabular}{@{} l ccccccc @{}}
    \toprule
    & \multicolumn{7}{c}{Wasserstein Distance ($\downarrow$)} \\
    \cmidrule(l){2-8}
    Prior Networks
      & $\sigma{=}0$ & $\sigma{=}0.14$ & $\sigma{=}0.22$ & $\sigma{=}0.37$
      & $\sigma{=}0.61$ & $\sigma{=}0.78$ & $\sigma{=}1.00$ \\
    \midrule
    Behavior Prior
      & \textbf{.0085} & \textbf{.0054} & \textbf{.0045} & \textbf{.0092}
      & \textbf{.0106} & \textbf{.0129} & \textbf{.0162} \\
    Motion Prior
      & .0116 & .0121 & .0115 & .0099
      & .0126 & .0171 & .0173 \\
    \bottomrule
  \end{tabular}
\end{table}
%%%%%%%%%%%%%%%%%%%%%%%%%%%%%%%%%%%%%%%%%%%%

%%%%%%%%%%%%%%%%%%%%%%%%%%%%%%%%%%%%%%%%%%%%%%%%%%%%%%%%%%%%%%%%%%%%%%%%%%%%%%%%%%%%%%%
\FloatBarrier
\subsection{Foraging task}

\subsubsection{Arena foraging supplemental results}

%%%%%%%%%%%%%%%%%%%%%%%%%%%%%%%%%%%%%%%%%%%%
% Supp 5: Arena forage rendering
%%%%%%%%%%%%%%%%%%%%%%%%%%%%%%%%%%%%%%%%%%%%
\begin{figure}[!htbp]
  \centering
  \includegraphics[width=0.75\textwidth]{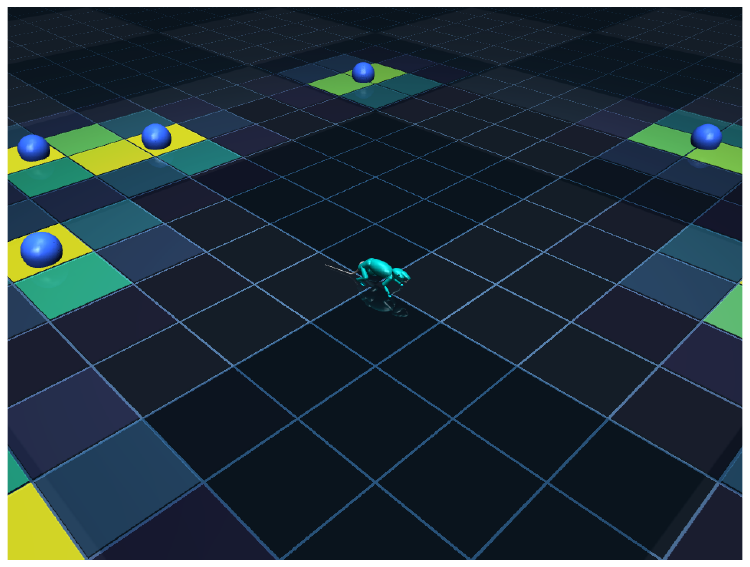}
  \caption{Rendering of the arena foraging task.}
  \label{fig:supp-foragerender}
\end{figure}
%%%%%%%%%%%%%%%%%%%%%%%%%%%%%%%%%%%%%%%%%%%%

%%%%%%%%%%%%%%%%%%%%%%%%%%%%%%%%%%%%%%%%%%%%
% Supp 6: Arena forage trajectories
%%%%%%%%%%%%%%%%%%%%%%%%%%%%%%%%%%%%%%%%%%%%
\begin{figure}[!htbp]
  \centering
  \includegraphics[width=0.9\textwidth]{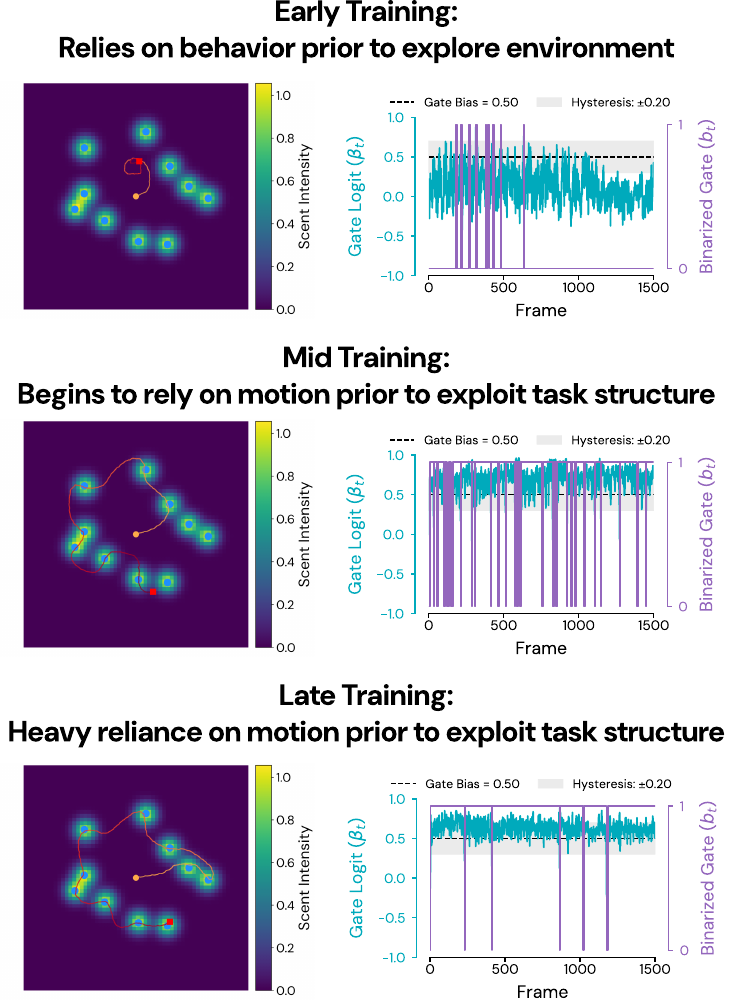}
  \caption{SCAMPER's training progression on the arena foraging task. Top-down view of trajectories the rodent takes along with the evolution of the gate logit ($\beta_t$) and its resulting binarization ($b_t$). Here when a treat is collected it is still displayed in the arena for visualization purposes, but in practice it disappear upon successful collection. $b_t=0$ results in stochastic sampling of the behavior prior. $b_t=1$ results in residual control of the motion prior. During early training the rodent relies heavily on the behavior prior to explore its environment. During mid training, the rodent begins to gate rapidly between both priors as it begins to exploit the reward signals from the collection of treats. During late training the rodent mostly uses residual control of the motion prior to effectively exploit the task structure.}
  \label{fig:supp-foragetraj}
\end{figure}
%%%%%%%%%%%%%%%%%%%%%%%%%%%%%%%%%%%%%%%%%%%%

\subsubsection{Alternatives to the adaptive binary gate}
\label{app:gate_ablation}

The adaptive binary gate is one way to decide when the behavior prior or the motion prior controls the latent. To test whether the specific mechanism matters, we replaced it in the foraging task with two alternatives that need no learned gate: a fixed exploration schedule and a fixed linear interpolation between the two priors. In both variants the policy outputs only the residual, and the motion prior, behavior prior and decoder stay frozen, as with the adaptive gate.

\paragraph{Fixed exploration schedule.} The latent is constructed exactly as with the adaptive gate, $\mathbf{z}_t = g_t\,(\boldsymbol{\mu}^{\mathrm{MP}}_t + \boldsymbol{\delta}_t) + (1 - g_t)\,\mathbf{z}^{\mathrm{BP}}_t$, where $\boldsymbol{\mu}^{\mathrm{MP}}_t$ is the motion prior mean, $\boldsymbol{\delta}_t$ is the policy's residual and $\mathbf{z}^{\mathrm{BP}}_t$ is a behavior prior sample. However, the binary gate $g_t$ is no longer emitted by the policy. Instead, it switches stochastically, and the probability $p(g_t = 1)$ of motion-prior control follows a fixed schedule over training. It is held at $0$ for an initial warmup fraction of training, ramped to $1$ over a subsequent ramp fraction, and held at $1$ thereafter. We vary when this switch happens.

\paragraph{Linear interpolation.} Instead of selecting one prior at each step, the two priors are blended with a constant weight $\rho \in [0, 1]$ and the residual is added on top: $\mathbf{z}_t = \rho\,\boldsymbol{\mu}^{\mathrm{MP}}_t + (1 - \rho)\,\mathbf{z}^{\mathrm{BP}}_t + \mathbf{r}_t$. Here $\rho$ is a fixed hyperparameter rather than a learned quantity. The endpoints $\rho = 1$ and $\rho = 0$ reduce to residual control of the motion prior alone and of the behavior prior alone, respectively.

\paragraph{Results.} Table~\ref{tab:gate_ablation} reports the mean number of treats collected in the hardest foraging condition (minimum treat radius of $1.5$\,m). Both alternatives can succeed, but only within a narrow range of settings. The fixed schedule trains successfully when control passes from the behavior prior to the motion prior early in training (warmup $0.05$, ramp $0.25$, so that the motion prior has full control by $30\%$ of training). It fails when the switch happens in the middle or later phases of training, and training also fails if the switch never happens. Likewise, interpolation succeeds only when the blend is weighted towards the behavior prior ($\rho = 0.25$) and fails for $\rho = 0.5$ and $\rho = 0.75$.

These results show that what matters is the right mix of the behavior and motion priors over training. The adaptive binary gate is one means of achieving that mix, not the only one. Its advantage over the alternatives is that the user does not need to hand-tune a schedule or blend weight: the gate only needs to be biased towards the behavior prior at the start of training, and the transition to the motion prior then emerges from learning (Figure~\ref{fig:supp-foragetraj}).

%%%%%%%%%%%%%%%%%%%%%%%%%%%%%%%%%%%%%%%%%%%%
% Supp Table: Alternatives to gating
%%%%%%%%%%%%%%%%%%%%%%%%%%%%%%%%%%%%%%%%%%%%
\begin{table}[h]
\centering
\caption{Alternatives to the adaptive binary gate in the foraging task. Values are the mean number of treats collected in the hardest condition (minimum treat radius $1.5$\,m). Schedule fractions are of total training steps. $\alpha$ is the weight on the motion prior.}
\label{tab:gate_ablation}
\begin{tabular}{llc}
\toprule
Variant & Setting & Mean treats collected \\
\midrule
\multirow{3}{*}{Fixed schedule (warmup / ramp)} & $0.05$ / $0.25$ & $9.94\pm0.03$ \\
 & $0.10$ / $0.60$ & $0.00\pm0.00$ \\
 & $0.25$ / $0.65$ & $0.00\pm0.00$ \\
\midrule
\multirow{3}{*}{Linear interpolation ($\rho$)} & $0.25$ & $9.91\pm0.04$ \\
 & $0.50$ & $0.00\pm0.00$ \\
 & $0.75$ & $0.00\pm0.00$ \\
\bottomrule
\end{tabular}
\end{table}
%%%%%%%%%%%%%%%%%%%%%%%%%%%%%%%%%%%%%%%%%%%%

%%%%%%%%%%%%%%%%%%%%%%%%%%%%%%%%%%%%%%%%%%%%%%%%%%%%%%%%%%%%%%%%%%%%%%%%%%%%%%%%%%%%%%%
\FloatBarrier
\subsection{Locomotion task}

Here we present the full gait analysis results for the empirical rat gait data \cite{danner_spinal_2023}, the main SCAMPER model which residually controls the motion prior (MP), a control which residually controls the behavior prior (BP), MIMIC-MJX, and the from-scratch trained model. Notably, when performing residual control using the behavior prior, galloping is never discovered and the lower speed gaits also lack resemblance to the empirical data both qualitatively (Fig.~\ref{fig:supp-gait}) and quantitatively (Table~\ref{tab:gait-full}).

%%%%%%%%%%%%%%%%%%%%%%%%%%%%%%%%%%%%%%%%%%%%
% Supp 7: Full gait plots
%%%%%%%%%%%%%%%%%%%%%%%%%%%%%%%%%%%%%%%%%%%%
\begin{figure}[!htbp]
  \centering
  \includegraphics[width=0.9\textwidth]{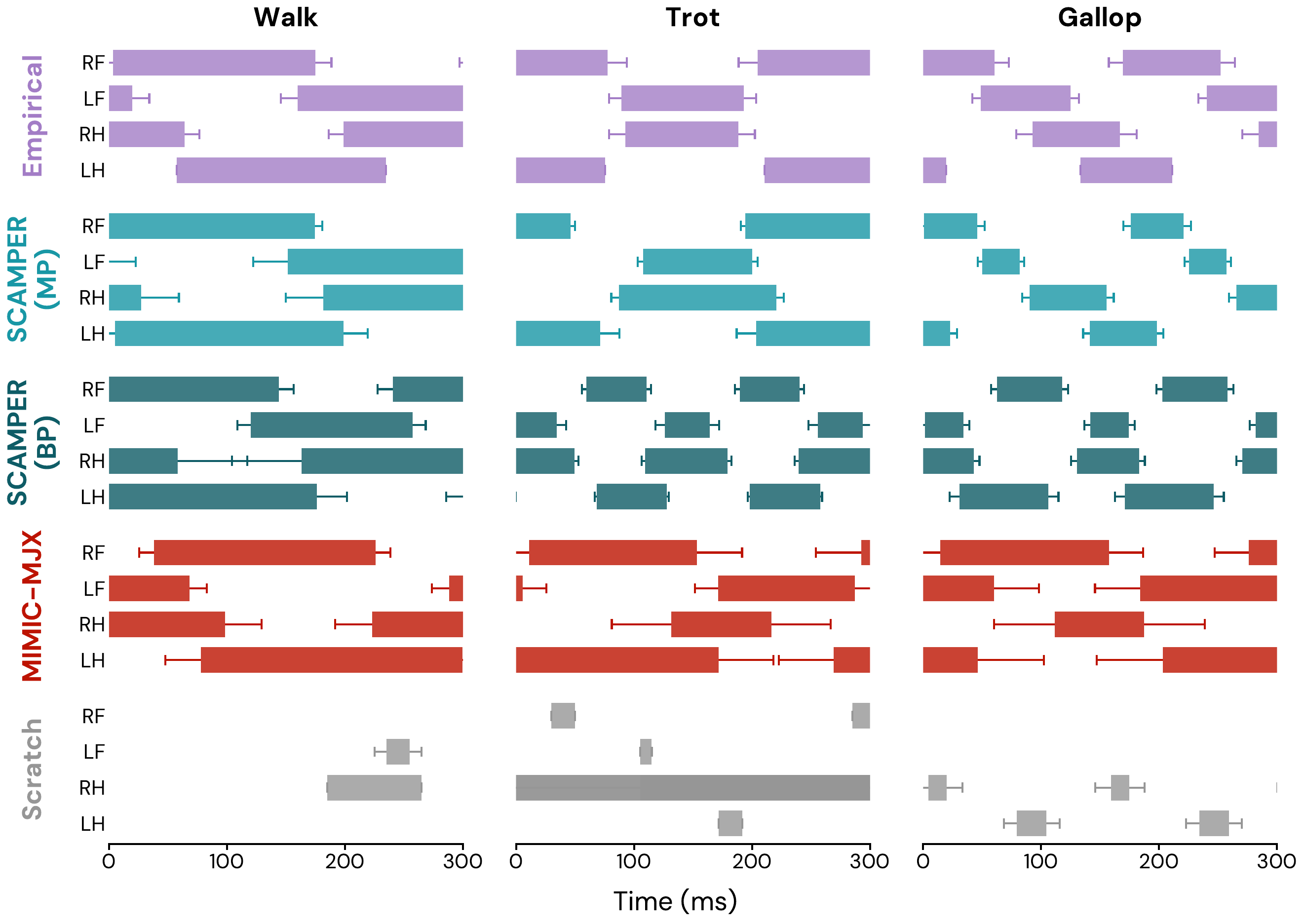}
  \caption{Stance-phase plots for each limb (RF, LF, RH, LH) across walk, trot, and gallop gaits, comparing empirical data (\textit{purple}; \cite{danner_spinal_2023}) to the base SCAMPER model (MP, \textit{blue}), the behavior-prior-controlled SCAMPER model (BP, \textit{dark blue}), MIMIC-MJX (\textit{red}), and the from-scratch model (\textit{gray}); bars show stance duration, error bars show variability across strides.}
  \label{fig:supp-gait}
\end{figure}
%%%%%%%%%%%%%%%%%%%%%%%%%%%%%%%%%%%%%%%%%%%%

%%%%%%%%%%%%%%%%%%%%%%%%%%%%%%%%%%%%%%%%%%%%
% Supp Table: Full gait analysis
%%%%%%%%%%%%%%%%%%%%%%%%%%%%%%%%%%%%%%%%%%%%
\begin{table}[!htbp]
\centering
\caption{Full gait metric comparison to empirical rat data. Empirical, SCAMPER (MP),
  SCAMPER (BP), MIMIC-MJX, and Scratch columns show mean $\pm$ std.
  $|\Delta\mu|/\sigma_{\text{emp}}$ is the absolute difference in means normalized by
  empirical standard deviation.}
\label{tab:gait-full}
\scriptsize
\setlength{\tabcolsep}{3pt}
\resizebox{\textwidth}{!}{%
\begin{tabular}{@{} l l r@{\,}c@{\,}l r@{\,}c@{\,}l r@{\,}c@{\,}l r@{\,}c@{\,}l
r@{\,}c@{\,}l cccc @{}}
  \toprule
  & & \multicolumn{3}{c}{Empirical} & \multicolumn{3}{c}{SCAMPER (MP)} &
\multicolumn{3}{c}{SCAMPER (BP)} & \multicolumn{3}{c}{MIMIC-MJX} &
\multicolumn{3}{c}{Scratch} & \multicolumn{4}{c}{$|\Delta\mu|/\sigma_{\text{emp}}$
($\downarrow$)} \\
  \cmidrule(lr){3-5} \cmidrule(lr){6-8} \cmidrule(lr){9-11} \cmidrule(lr){12-14}
\cmidrule(lr){15-17} \cmidrule(l){18-21}
  Gait & Metric & \multicolumn{3}{c}{} & \multicolumn{3}{c}{} & \multicolumn{3}{c}{} &
\multicolumn{3}{c}{} & \multicolumn{3}{c}{} & SCAMPER (MP) & SCAMPER (BP) & MIMIC-MJX &
Scratch \\
  \midrule
  %% ---- WALK ----
  \multirow{21}{*}{\rotatebox[origin=c]{90}{Walk}}
  & Stride Duration (ms) & 308 &$\pm$& 91 & 332 &$\pm$& 16 & 313 &$\pm$& 34 & 374 &$\pm$&
23 & 255 &$\pm$& 293 & 0.27 & \textbf{0.06} & 0.73 & 0.58 \\
  & Step Frequency (Hz) & 3.46 &$\pm$& 0.75 & 3.02 &$\pm$& 0.15 & 3.22 &$\pm$& 0.33 & 2.69
&$\pm$& 0.26 & 4.61 &$\pm$& 2.62 & 0.58 & \textbf{0.31} & 1.02 & 1.53 \\
  & Stance LH (ms) & 182 &$\pm$& 82 & 191 &$\pm$& 21 & 178 &$\pm$& 26 & 221 &$\pm$& 36 &
223 &$\pm$& 295 & 0.11 & \textbf{0.05} & 0.48 & 0.50 \\
  & Stance RH (ms) & 175 &$\pm$& 66 & 183 &$\pm$& 9 & 208 &$\pm$& 59 & 247 &$\pm$& 31 & 80
&$\pm$& 0 & \textbf{0.12} & 0.51 & 1.09 & 1.45 \\
  & Stance LF (ms) & 171 &$\pm$& 71 & 172 &$\pm$& 32 & 137 &$\pm$& 11 & 153 &$\pm$& 15 &
20 &$\pm$& 10 & \textbf{0.02} & 0.48 & 0.25 & 2.12 \\
  & Stance RF (ms) & 174 &$\pm$& 69 & 183 &$\pm$& 6 & 217 &$\pm$& 13 & 187 &$\pm$& 16 & 0
&$\pm$& 0 & \textbf{0.13} & 0.62 & 0.19 & 2.53 \\
  & Swing LH (ms) & 125 &$\pm$& 20 & 144 &$\pm$& 5 & 128 &$\pm$& 7 & 152 &$\pm$& 10 & 150
&$\pm$& 45 & 0.96 & \textbf{0.11} & 1.33 & 1.23 \\
  & Swing RH (ms) & 133 &$\pm$& 31 & 148 &$\pm$& 5 & 113 &$\pm$& 21 & 123 &$\pm$& 8 & 175
&$\pm$& 0 & 0.51 & 0.65 & \textbf{0.29} & 1.37 \\
  & Swing LF (ms) & 136 &$\pm$& 25 & 152 &$\pm$& 7 & 174 &$\pm$& 9 & 219 &$\pm$& 22 & 120
&$\pm$& 0 & \textbf{0.63} & 1.52 & 3.33 & 0.66 \\
  & Swing RF (ms) & 133 &$\pm$& 30 & 149 &$\pm$& 3 & 94 &$\pm$& 5 & 184 &$\pm$& 7 & 255
&$\pm$& 0 & \textbf{0.50} & 1.30 & 1.67 & 3.99 \\
  & Duty Factor LH & .576 &$\pm$& .074 & .576 &$\pm$& .064 & .568 &$\pm$& .082 & .593
&$\pm$& .095 & .877 &$\pm$& 1.157 & \textbf{0.00} & 0.11 & 0.23 & 4.08 \\
  & Duty Factor RH & .562 &$\pm$& .055 & .550 &$\pm$& .027 & .665 &$\pm$& .188 & .660
&$\pm$& .083 & .314 &$\pm$& .000 & \textbf{0.21} & 1.87 & 1.77 & 4.47 \\
  & Duty Factor LF & .544 &$\pm$& .059 & .519 &$\pm$& .098 & .438 &$\pm$& .036 & .409
&$\pm$& .039 & .079 &$\pm$& .039 & \textbf{0.44} & 1.82 & 2.31 & 7.96 \\
  & Duty Factor RF & .558 &$\pm$& .055 & .552 &$\pm$& .020 & .691 &$\pm$& .041 & .502
&$\pm$& .042 & .000 &$\pm$& .000 & \textbf{0.11} & 2.42 & 1.03 & 10.14 \\
  & Temporal Symmetry & .011 &$\pm$& .068 & .016 &$\pm$& .024 & $-$.050 &$\pm$& .067 &
$-$.049 &$\pm$& .074 & .000 &$\pm$& .000 & \textbf{0.08} & 0.89 & 0.88 & 0.16 \\
  & Phase RH rel.\ LH & .459 &$\pm$& .041 & .513 &$\pm$& .113 & .539 &$\pm$& .104 & .368
&$\pm$& .094 & .432 &$\pm$& .000 & 1.32 & 1.96 & 2.25 & \textbf{0.67} \\
  & Phase LF rel.\ LH & .333 &$\pm$& .047 & .459 &$\pm$& .094 & .428 &$\pm$& .062 & .516
&$\pm$& .102 & .334 &$\pm$& .295 & 2.69 & 2.04 & 3.93 & \textbf{0.02} \\
  & Phase RF rel.\ LH & .824 &$\pm$& .045 & .906 &$\pm$& .165 & .755 &$\pm$& .077 & .835
&$\pm$& .206 & .000 &$\pm$& .000 & 1.86 & 1.53 & \textbf{0.26} & 18.48 \\
  \cmidrule(l){2-21}
  & Mean (Coordination) & \multicolumn{3}{c}{} & \multicolumn{3}{c}{} &
\multicolumn{3}{c}{} & \multicolumn{3}{c}{} & \multicolumn{3}{c}{} & \textbf{1.49} & 1.60 &
1.83 & 4.83 \\
  & Mean (Gait Pattern) & \multicolumn{3}{c}{} & \multicolumn{3}{c}{} &
\multicolumn{3}{c}{} & \multicolumn{3}{c}{} & \multicolumn{3}{c}{} & \textbf{0.95} & 1.68 &
1.68 & 6.55 \\
  & Mean (All) & \multicolumn{3}{c}{} & \multicolumn{3}{c}{} & \multicolumn{3}{c}{} &
\multicolumn{3}{c}{} & \multicolumn{3}{c}{} & \textbf{0.59} & 1.01 & 1.28 & 3.44 \\
  \midrule
  %% ---- TROT ----
  \multirow{21}{*}{\rotatebox[origin=c]{90}{Trot}}
  & Stride Duration (ms) & 230 &$\pm$& 31 & 267 &$\pm$& 10 & 130 &$\pm$& 4 & 285 &$\pm$&
92 & 318 &$\pm$& 391 & \textbf{1.19} & 3.24 & 1.76 & 2.84 \\
  & Step Frequency (Hz) & 4.42 &$\pm$& 0.56 & 3.75 &$\pm$& 0.22 & 7.71 &$\pm$& 0.26 & 4.74
&$\pm$& 2.81 & 4.41 &$\pm$& 2.52 & 1.18 & 5.84 & 0.57 & \textbf{0.01} \\
  & Stance LH (ms) & 97 &$\pm$& 27 & 135 &$\pm$& 16 & 60 &$\pm$& 2 & 181 &$\pm$& 48 & 20
&$\pm$& 0 & 1.46 & \textbf{1.38} & 3.18 & 2.88 \\
  & Stance RH (ms) & 97 &$\pm$& 25 & 134 &$\pm$& 6 & 70 &$\pm$& 3 & 87 &$\pm$& 50 & 1251
&$\pm$& 2321 & 1.49 & 1.10 & \textbf{0.40} & 46.48 \\
  & Stance LF (ms) & 104 &$\pm$& 23 & 93 &$\pm$& 5 & 38 &$\pm$& 8 & 116 &$\pm$& 18 & 10
&$\pm$& 0 & 0.51 & 2.93 & \textbf{0.50} & 4.16 \\
  & Stance RF (ms) & 104 &$\pm$& 24 & 118 &$\pm$& 5 & 51 &$\pm$& 4 & 155 &$\pm$& 27 & 20
&$\pm$& 0 & \textbf{0.60} & 2.25 & 2.15 & 3.56 \\
  & Swing LH (ms) & 134 &$\pm$& 14 & 129 &$\pm$& 3 & 70 &$\pm$& 3 & 135 &$\pm$& 26 & 110
&$\pm$& 10 & 0.33 & 4.64 & \textbf{0.06} & 1.72 \\
  & Swing RH (ms) & 133 &$\pm$& 15 & 135 &$\pm$& 5 & 60 &$\pm$& 3 & 80 &$\pm$& 14 & 190
&$\pm$& 43 & \textbf{0.10} & 4.98 & 3.67 & 3.88 \\
  & Swing LF (ms) & 126 &$\pm$& 16 & 175 &$\pm$& 5 & 92 &$\pm$& 9 & 211 &$\pm$& 22 & 308
&$\pm$& 0 & 3.13 & \textbf{2.18} & 5.44 & 11.69 \\
  & Swing RF (ms) & 126 &$\pm$& 16 & 150 &$\pm$& 4 & 79 &$\pm$& 4 & 173 &$\pm$& 7 & 298
&$\pm$& 0 & \textbf{1.47} & 2.93 & 2.89 & 10.62 \\
  & Duty Factor LH & .413 &$\pm$& .065 & .506 &$\pm$& .061 & .460 &$\pm$& .012 & .636
&$\pm$& .169 & .063 &$\pm$& .000 & 1.43 & \textbf{0.72} & 3.43 & 5.42 \\
  & Duty Factor RH & .416 &$\pm$& .060 & .502 &$\pm$& .024 & .537 &$\pm$& .025 & .306
&$\pm$& .176 & 3.931 &$\pm$& 7.291 & \textbf{1.42} & 2.02 & 1.85 & 58.61 \\
  & Duty Factor LF & .450 &$\pm$& .053 & .348 &$\pm$& .017 & .292 &$\pm$& .061 & .407
&$\pm$& .064 & .031 &$\pm$& .000 & 1.95 & 3.01 & \textbf{0.83} & 7.96 \\
  & Duty Factor RF & .448 &$\pm$& .055 & .442 &$\pm$& .017 & .391 &$\pm$& .030 & .543
&$\pm$& .094 & .063 &$\pm$& .000 & \textbf{0.11} & 1.02 & 1.71 & 6.95 \\
  & Temporal Symmetry & $-$.004 &$\pm$& .058 & .008 &$\pm$& .068 & $-$.075 &$\pm$& .017 &
.172 &$\pm$& .233 & $-$.434 &$\pm$& .162 & \textbf{0.21} & 1.23 & 3.06 & 7.46 \\
  & Phase RH rel.\ LH & .488 &$\pm$& .060 & .543 &$\pm$& .092 & .308 &$\pm$& .045 & .474
&$\pm$& .210 & .401 &$\pm$& .347 & 0.91 & 2.98 & \textbf{0.24} & 1.45 \\
  & Phase LF rel.\ LH & .473 &$\pm$& .046 & .620 &$\pm$& .075 & .406 &$\pm$& .044 & .699
&$\pm$& .121 & .000 &$\pm$& .000 & 3.21 & \textbf{1.47} & 4.94 & 10.35 \\
  & Phase RF rel.\ LH & .975 &$\pm$& .070 & .921 &$\pm$& .173 & .900 &$\pm$& .142 & .108
&$\pm$& .070 & .126 &$\pm$& .000 & \textbf{0.77} & 1.07 & 12.36 & 12.11 \\
  \cmidrule(l){2-21}
  & Mean (Coordination) & \multicolumn{3}{c}{} & \multicolumn{3}{c}{} &
\multicolumn{3}{c}{} & \multicolumn{3}{c}{} & \multicolumn{3}{c}{} & \textbf{1.27} & 1.69 &
5.15 & 7.84 \\
  & Mean (Gait Pattern) & \multicolumn{3}{c}{} & \multicolumn{3}{c}{} &
\multicolumn{3}{c}{} & \multicolumn{3}{c}{} & \multicolumn{3}{c}{} & \textbf{1.40} & 1.76 &
3.62 & 14.69 \\
  & Mean (All) & \multicolumn{3}{c}{} & \multicolumn{3}{c}{} & \multicolumn{3}{c}{} &
\multicolumn{3}{c}{} & \multicolumn{3}{c}{} & \textbf{1.19} & 2.50 & 2.72 & 11.01 \\
  \midrule
  %% ---- GALLOP ----
  \multirow{21}{*}{\rotatebox[origin=c]{90}{Gallop}}
  & Stride Duration (ms) & 192 &$\pm$& 19 & 175 &$\pm$& 6 & 140 &$\pm$& 5 & 255 &$\pm$& 91
& 172 &$\pm$& 100 & \textbf{0.85} & 2.67 & 3.32 & 1.01 \\
  & Step Frequency (Hz) & 5.27 &$\pm$& 0.53 & 5.71 &$\pm$& 0.23 & 7.15 &$\pm$& 0.25 & 5.22
&$\pm$& 2.54 & 7.95 &$\pm$& 3.11 & 0.83 & 3.52 & \textbf{0.09} & 5.04 \\
  & Stance LH (ms) & 78 &$\pm$& 12 & 56 &$\pm$& 6 & 75 &$\pm$& 9 & 88 &$\pm$& 53 & 15
&$\pm$& 8 & 1.83 & \textbf{0.22} & 0.81 & 5.32 \\
  & Stance RH (ms) & 74 &$\pm$& 13 & 66 &$\pm$& 6 & 53 &$\pm$& 5 & 75 &$\pm$& 50 & 16
&$\pm$& 17 & 0.62 & 1.61 & \textbf{0.02} & 4.33 \\
  & Stance LF (ms) & 76 &$\pm$& 11 & 31 &$\pm$& 4 & 32 &$\pm$& 5 & 138 &$\pm$& 40 & 0
&$\pm$& 0 & 4.00 & \textbf{3.92} & 5.49 & 6.80 \\
  & Stance RF (ms) & 84 &$\pm$& 16 & 47 &$\pm$& 6 & 55 &$\pm$& 5 & 141 &$\pm$& 28 & 0
&$\pm$& 0 & 2.30 & \textbf{1.76} & 3.57 & 5.20 \\
  & Swing LH (ms) & 114 &$\pm$& 13 & 119 &$\pm$& 3 & 65 &$\pm$& 8 & 115 &$\pm$& 46 & 172
&$\pm$& 88 & 0.45 & 3.76 & \textbf{0.12} & 4.52 \\
  & Swing RH (ms) & 117 &$\pm$& 10 & 108 &$\pm$& 6 & 87 &$\pm$& 5 & 93 &$\pm$& 13 & 140
&$\pm$& 103 & \textbf{0.86} & 2.93 & 2.40 & 2.24 \\
  & Swing LF (ms) & 115 &$\pm$& 11 & 144 &$\pm$& 5 & 108 &$\pm$& 5 & 189 &$\pm$& 30 & 172
&$\pm$& 0 & 2.55 & \textbf{0.66} & 6.53 & 5.07 \\
  & Swing RF (ms) & 108 &$\pm$& 9 & 128 &$\pm$& 6 & 85 &$\pm$& 5 & 187 &$\pm$& 9 & 172
&$\pm$& 0 & \textbf{2.16} & 2.45 & 8.42 & 6.82 \\
  & Duty Factor LH & .406 &$\pm$& .042 & .321 &$\pm$& .033 & .538 &$\pm$& .061 & .343
&$\pm$& .207 & .087 &$\pm$& .044 & 2.04 & 3.14 & \textbf{1.52} & 7.63 \\
  & Duty Factor RH & .386 &$\pm$& .040 & .377 &$\pm$& .037 & .377 &$\pm$& .034 & .292
&$\pm$& .197 & .094 &$\pm$& .098 & \textbf{0.23} & 0.24 & 2.37 & 7.36 \\
  & Duty Factor LF & .398 &$\pm$& .032 & .179 &$\pm$& .025 & .230 &$\pm$& .037 & .540
&$\pm$& .157 & .000 &$\pm$& .000 & 6.75 & 5.17 & \textbf{4.41} & 12.28 \\
  & Duty Factor RF & .433 &$\pm$& .048 & .266 &$\pm$& .036 & .394 &$\pm$& .036 & .552
&$\pm$& .110 & .000 &$\pm$& .000 & 3.47 & \textbf{0.82} & 2.46 & 8.99 \\
  & Temporal Symmetry & .025 &$\pm$& .059 & $-$.084 &$\pm$& .060 & .173 &$\pm$& .062 &
.086 &$\pm$& .311 & $-$.095 &$\pm$& .254 & 1.83 & 2.49 & \textbf{1.02} & 2.03 \\
  & Phase RH rel.\ LH & .789 &$\pm$& .073 & .692 &$\pm$& .084 & .697 &$\pm$& .102 & .446
&$\pm$& .183 & .540 &$\pm$& .298 & 1.33 & \textbf{1.26} & 4.73 & 3.44 \\
  & Phase LF rel.\ LH & .560 &$\pm$& .037 & .486 &$\pm$& .034 & .790 &$\pm$& .047 & .495
&$\pm$& .211 & .000 &$\pm$& .000 & 1.99 & 6.26 & \textbf{1.76} & 15.20 \\
  & Phase RF rel.\ LH & .188 &$\pm$& .062 & .202 &$\pm$& .044 & .227 &$\pm$& .044 & .168
&$\pm$& .037 & .000 &$\pm$& .000 & \textbf{0.23} & 0.64 & 0.32 & 3.05 \\
  \cmidrule(l){2-21}
  & Mean (Coordination) & \multicolumn{3}{c}{} & \multicolumn{3}{c}{} &
\multicolumn{3}{c}{} & \multicolumn{3}{c}{} & \multicolumn{3}{c}{} & \textbf{1.35} & 2.66 &
1.96 & 5.93 \\
  & Mean (Gait Pattern) & \multicolumn{3}{c}{} & \multicolumn{3}{c}{} &
\multicolumn{3}{c}{} & \multicolumn{3}{c}{} & \multicolumn{3}{c}{} & \textbf{2.29} & 2.50 &
2.51 & 8.28 \\
  & Mean (All) & \multicolumn{3}{c}{} & \multicolumn{3}{c}{} & \multicolumn{3}{c}{} &
\multicolumn{3}{c}{} & \multicolumn{3}{c}{} & \textbf{1.91} & 2.42 & 2.74 & 5.91 \\
  \bottomrule
\end{tabular}%
}
\end{table}
%%%%%%%%%%%%%%%%%%%%%%%%%%%%%%%%%%%%%%%%%%%%

\end{document}